\documentclass{article}
\usepackage{kr}

\usepackage{times}
\usepackage{soul}
\usepackage{url}
\usepackage[hidelinks]{hyperref}
\usepackage[utf8]{inputenc}
\usepackage[small]{caption}
\usepackage{graphicx}
\usepackage{amsmath}
\usepackage{amsthm}
\usepackage{booktabs}
\usepackage{amsfonts} 
\usepackage{xcolor} 
\usepackage{cleveref} 
\usepackage[noend]{algorithm2e} 
\usepackage{tikz} 
\usetikzlibrary{positioning}
\usepackage[inline]{enumitem} 
\usepackage{amssymb} 
\usepackage{enumitem} 

\SetKwInOut{Input}{Input} 
\SetKwInOut{Output}{Output} 
\SetKwComment{Comment}{/* }{ */} 
\DontPrintSemicolon 
\crefname{algocf}{alg.}{algs.} 
\Crefname{algocf}{Algorithm}{Algorithms}
\RestyleAlgo{ruled} 
\LinesNumbered 

\newtheorem{theorem}{Theorem}
\newtheorem*{theorem*}{Theorem}

\newtheorem{proposition}[theorem]{Proposition}
\newtheorem*{proposition*}{Proposition}
\newtheorem{lemma}[theorem]{Lemma}
\newtheorem*{lemma*}{Lemma}
\newtheorem{definition}[theorem]{Definition}

\definecolor{bluecite}{HTML}{0875b7}
\hypersetup{
  colorlinks=true,
  citecolor=bluecite,
  linkcolor=magenta
}

\newcommand{\col}{\text{Col}}

\newcommand{\changemarker}[1]{#1} 

\title{Relational Graph Convolutional Networks Do Not Learn Sound Rules}

\author{%
Matthew Morris$^1$\and
David J. Tena Cucala$^{1,2}$\and
Bernardo Cuenca Grau$^1$\and
Ian Horrocks$^1$ \\
\affiliations
$^1$Department of Compute Science, University of Oxford\\
$^2$Department of Computer Science, Royal Holloway, University of London
\emails
\{matthew.morris, bernardo.cuenca.grau, ian.horrocks\}@cs.ox.ac.uk,
david.tenacucala@rhul.ac.uk
}

\begin{document}

\maketitle

\begin{abstract}
Graph neural networks (GNNs) are frequently used to predict missing facts in knowledge graphs (KGs). Motivated by the lack of explainability for the outputs of these models, recent work has aimed to explain their predictions using Datalog, a widely used logic-based formalism. However, such work has been restricted to certain subclasses of GNNs. In this paper, we consider one of the most popular GNN architectures for KGs, R-GCN, and we provide two methods to extract rules that explain its predictions \changemarker{and are \emph{sound}, in the sense that each fact derived by the rules is also predicted by the GNN, for any input dataset}. Furthermore, we provide a method that can verify that certain classes of Datalog rules are not sound for the R-GCN. In our experiments, we train R-GCNs on KG completion benchmarks, and we are able to verify that no Datalog rule is sound for these models, even though the models often obtain high to near-perfect accuracy. This raises some concerns about the ability of R-GCN models to generalise and about the explainability of their predictions. We further provide two variations to the training paradigm of R-GCN that encourage it to learn sound rules and find a trade-off between model accuracy and the number of learned sound rules.
\end{abstract}
\section{Introduction}

Knowledge graphs (KGs) are graph-structured knowledge bases where nodes and edges represent entities and their relationships, respectively \cite{DBLP:journals/csur/HoganBCdMGKGNNN21}. KGs can be typically stored as sets of unary and binary facts, and are being exploited in an increasing range of applications \cite{vrandevcic2014wikidata,suchanek2007yago,bouchard2015approximate,hamilton2017inductive}. 

Knowledge graphs are, however, often incomplete, which has led to the rapid development of the field of KG completion---the task of extending an incomplete KG with all missing facts holding in its (unknown) complete version. KG completion is typically conceptualised as a classification problem, where the aim is to learn a function that, given the incomplete KG and a candidate fact as input, decides whether the latter holds in the  completion of the former. A wide range of KG approaches have been proposed in the literature, including embedding-based approaches with distance-based scoring functions \cite{bordes2013translating,sun2018rotate}, tensor products \cite{nickel2011three,yang2015embedding}, 
box embeddings \cite{DBLP:conf/nips/AbboudCLS20},
recurrent neural networks \cite{sadeghian2019drum}, differentiable reasoning \cite{rocktaschel2017end,evans2018learning}, and LLMs \cite{yao2019kg,xie2022discrimination}.
Amongst all neural approaches to KG completion, however, those based on graph neural networks (GNNs) have received special attention \cite{ioannidis2019recurrent,liu2021indigo,pflueger2022gnnq}. These include R-GCN \cite{schlichtkrull2018modeling} and its  extensions \cite{tian2020ra,cai2019transgcn,vashishth2019composition,yu2021knowledge,shang2019end,liu2021indigo}, where the basic R-GCN model remains a common baseline for evaluating against or as part of a larger system \cite{gutteridge2023drew,liu2023revisiting,li2023skier,tang2024gadbench,zhang2023learning,lin2023fusing}.

Although these embedding and neural-based approaches to KG completion have proved effective in practice, 
their predictions are difficult to explain and interpret \cite{garnelo2019reconciling}. This is in contrast to logic-based and neuro-symbolic approaches to KG completion, such as rule learning methods, where the extracted rules can be used to generate rigourous proofs explaining the prediction of any given fact.
RuleN \cite{meilicke2018fine} and AnyBURL \cite{meilicke2018fine} heuristically identify Datalog rules from the given data and apply them directly for completing the input graph.  RNNLogic \cite{qu2020rnnlogic} uses a probabilistic model to select the most promising rules. Other works attempt to extract Datalog rules from trained neural models, including Neural-LP \cite{yang2017differentiable}, DRUM \cite{sadeghian2019drum}, and Neural Theorem Provers \cite{rocktaschel2017end}. As shown in \cite{cucala2022faithful,wang2023faithful}, however, the extracted Datalog rules are not faithful to the model in the sense that the rules may derive, for some dataset, facts that are not predicted by the model (\emph{unsoundness}) as well as failing to derive predicted facts (\emph{incompleteness}).

As a result, there is increasing interest in the development of neural KG completion methods whose predictions can be faithfully characterised by means of rule-based reasoning. As shown in \cite{cucala2022faithful,wang2023faithful}, the Neural-LP and DRUM approaches can be adapted to ensure soundness and completeness of the generated rules with respect to the model. The class of \emph{monotonic GNNs} \cite{cucala2021explainable}, which use max aggregation, restrict all weights in the model to be non-negative, and impose certain requirements on activation and classification functions, can be faithfully captured by means of tree-like Datalog programs.  In turn, \emph{monotonic max-sum GNNs} \cite{cucala2023correspondence}, which relax the requirements on the aggregation function to encompass both max and sum aggregation, can be faithfully characterised by means of tree-like Datalog programs with inequalities in the body of rules. Both variants of monotonic GNNs, however, require model weights to be non-negative, which is crucial to ensure that their application to datasets is monotonic under homomorphisms---a key property of (negation-free) Datalog reasoning.
The monotonicity requirement, however, 
is not applicable to popular GNN-based models such as R-GCN.



\paragraph{Our Contribution}
In this paper, we consider \emph{sum-GNNs}, which use sum as aggregation function and which do not impose restrictions on model weights. These GNNs can be seen both as an extension to the GNNs in \cite{cucala2023correspondence} with sum aggregation but without the monotonicity requirement, as well as a variant of the R-GCN model. The functions learnt by sum-GNNs may be non-monotonic: predicted facts may be invalidated by adding new facts to the input KG. As a result, these GNNs cannot, in general, be faithfully captured by (negation-free) Datalog programs. 
Our aim in this paper is to identify a subset of the output channels \changemarker{(i.e.\ features)} of the GNN which exhibit monotonic behaviour, and for which sound Datalog rules may be extracted. The ability to extract sound rules is important as it allows us to explain model predictions associated to the identified output channels. Furthermore, we provide means for identifying \emph{unbounded} output channels for which no sound Datalog rule exists, hence suggesting that these channels inherently exhibit non-monotonic behaviour.

We conducted experiments on the benchmark datasets by \cite{teru2020inductive} and also use the rule-based LogInfer evaluation framework described in \cite{liu2023revisiting}, which we extended to include a mixture of monotonic and non-monotonic rules. Our experiments show that even under ideal scenarios, without restrictions on the training process, all the channels in the trained GNN are unbounded (even for monotonic LogInfer benchmarks), which implies that there are no sound Datalog rules for the model. We then consider two adjustments to the training process where weights sufficiently close to zero (as specified by a given threshold) are clamped to zero iteratively during training. As the weight clamping threshold increases, we observe that an increasing number of output channels exhibit monotonic behaviour and we obtain more sound rules, although the model accuracy diminishes. Hence, there is a trade-off between model performance and rule extraction.
\section{Background} \label{sec:background}

\smallskip
\noindent
\textbf{Datalog}
We fix a signature of countably infinite, disjoint sets of predicates and
constants, where each predicate is associated with a non-negative integer arity. We also consider a
countably infinite set of variables disjoint with the sets of
predicates and constants. 
A term is a variable or a constant. An atom is an expression of the form $R(t_1, ..., t_n)$, where each $t_i$ is a term and $R$ is a predicate with arity $n$. 
A literal is an atom or any inequality $t_1 \not\approx t_2$. A literal is ground if it contains no variables. A fact is a ground atom and a dataset $D$ is a finite set of facts. A (Datalog) rule is an expression of the form
\begin{equation}
    B_1 \land ... \land B_n \rightarrow H, 
\label{eq:ruleform}
\end{equation} where  $B_1, ..., B_n$ are its body literals and $H$ is its head atom.
We use the standard safety requirements for rules: every variable that appears in a rule must occur in a body atom.
Furthermore, to avoid vacuous rules, we require that each inequality in the body of a rule mentions two different terms.
A (Datalog) program is a finite set of rules. A substitution $\nu$ maps finitely many variables to constants. For literal $\alpha$ and a substitution $\nu$ defined on each variable in $\alpha$, $\alpha \nu$ is obtained by replacing each occurrence of a variable $x$ in $\alpha$ with $\nu(x)$. For a dataset $D$ and a ground atom $B$, we write $D \models B$ if $B \in D$; furthermore, given constants $a_1$ and $a_2$, we write $D \models a_1 \not\approx a_2$ if $a_1 \neq a_2$, for uniformity. The immediate consequence operator $T_r$ for a rule $r$ of form  
\eqref{eq:ruleform} maps a dataset $D$ to dataset $T_r(D)$ containing $H \nu$ for each substitution $\nu$ such that $D \models B_i \nu$ for each $i \in \{1, \dots, n\}$. For a program $\mathcal{P}$, $T_{\mathcal{P}}(D) = \bigcup_{r \in \mathcal{P}} T_r(D)$.

\smallskip
\noindent
\textbf{Graphs}
We consider real-valued vectors and matrices. For $\mathbf{v}$ a vector and ${i > 0}$, $\mathbf{v}[i]$ denotes the $i$-th element of $\mathbf{v}$. For $\mathbf{A}$ a
matrix and ${i ,j> 0}$, $\mathbf{A}[i,j]$ denotes the element in row $i$ and column
$j$ of $\mathbf{A}$. A function ${\sigma : \mathbb{R} \to \mathbb{R}}$ is \emph{monotonically increasing} if
${x < y}$ implies ${\sigma(x) \leq \sigma(y)}$.
We apply functions to vectors element-wise.

For a finite set $\col$ of colours  and $\delta \in \mathbb{N}$, a ($\col$, $\delta$)-graph $G$ is a tuple $\langle V, \{ E^c \}_{c \in \text{$\col$}}, \lambda \rangle$ where $V$ is a finite vertex set, each $E^c \subseteq V \times V$ is a set of directed edges, and $\lambda$ assigns to each $v \in V$ a vector of dimension $\delta$. 
When $\lambda$ is clear from the context, we abbreviate the \changemarker{labelling} $\lambda(v)$ as $\mathbf{v}$.
Graph $G$ is undirected if $E^c$ is symmetric for each $c \in$ $\col$ and is Boolean if $\mathbf{v}[i] \in \{0, 1\}$ for each $v \in V$ and  $i \in \{ 1, ..., \delta \}$.

\smallskip
\noindent
\textbf{Graph Neural Networks}
We consider GNNs with sum aggregation.
A ($\col$,$\delta$)-sum graph neural network (sum-GNN) $\mathcal{N}$
with ${L \geq 1}$ layers is a tuple
\begin{equation}
\resizebox{.86\linewidth}{!}{$
    \langle \; \{ \mathbf{A}_\ell \}_{1 \leq \ell \leq L}, \; \{ \mathbf{B}^c_\ell\}_{c \in \text{$\col$}, 1 \leq \ell \leq L}, \; \{ \mathbf{b}_\ell \}_{1 \leq \ell \leq L}, \; \sigma_\ell, \; \texttt{cls}_t \; \rangle,  \label{eq:GNN}
$}
\end{equation}
where, for each ${\ell \in \{ 1, \dots, L \}}$ and ${c \in \col}$, matrices
$\mathbf{A}_\ell$ and $\mathbf{B}^c_\ell$ are of dimension
$\delta_\ell \times \delta_{\ell-1}$ with ${\delta_0 = \delta_L = \delta}$,
$\mathbf{b}_\ell$ is a vector of dimension $\delta_\ell$, $\sigma_\ell : \mathbb{R} \to \mathbb{R}^+ \cup \{ 0 \}$ is a monotonically increasing activation function with non-negative range, 
and ${\texttt{cls}_t : \mathbb{R} \to \{ 0,1 \}}$ for threshold $t \in \mathbb{R}$ is a step classification function such that $\texttt{cls}_t(x) = 1$ if  $ x \geq t$ and $\texttt{cls}_t(x) = 0$ otherwise.

Applying $\mathcal{N}$ to a $(\col,\delta)$-graph induces a sequence of labels $\mathbf{v}_0, \mathbf{v}_1, ..., \mathbf{v}_L$ for each vertex $v$ in the graph as follows. First,
$\mathbf{v}_0$ is the initial labelling of the input graph; then, for each $1 \leq \ell \leq L$, $\mathbf{v}_\ell$ is defined by the following expression:
\begin{align} \label{align:def_gnn_update}
\mathbf{v}_\ell = \sigma_\ell ( \mathbf{b}_\ell + \mathbf{A}_\ell \mathbf{v}_{\ell - 1} + \sum_{c \in \text{$\col$}} \mathbf{B}_\ell^c \sum_{(v, u) \in E^c} \mathbf{u}_{\ell - 1})
\end{align}
The output of $\mathcal{N}$ is a (\col,$\delta$)-graph with the same vertices and edges as the input graph, but where each vertex is labelled by $\texttt{cls}_t(\mathbf{v}_L)$. For layer $\ell \in \{ 0, ..., L \}$ of $\mathcal{N}$, each $i \in \{ 1, ..., \delta_\ell \}$ is referred to as a \emph{channel}.

The R-GCN model \cite{schlichtkrull2018modeling} can be seen as a sum-GNN variant with ReLU activations and zero biases in all layers. The definition in  
\cite{schlichtkrull2018modeling} includes a normalisation parameter $c_{i, r}$ (cf.\ their Equation (2)), which can depend on a predicate $r$ and/or vertex $i$ under special consideration when applying the GNN. A dependency on the vertex implies that the values of these parameters are data-dependent---that is, they are computed at test time based on the concrete graph over which the GNN is evaluated (e.g.,  $c_{i, r}$ could be computed as the number of $r$-neighbours of vertex $i$).
Sum-GNNs capture the R-GCN model under the assumption that the normalisation parameters are data-independent and hence can be fixed after training; we further assume for simplicity that they are set to $1$.

\smallskip
\noindent
\textbf{Dataset Transformations Through sum-GNNs}
A sum-GNN $\mathcal{N}$ can be used to realise a transformation
$T_{\mathcal{N}}$ from datasets to datasets over a given finite signature \cite{cucala2023correspondence}. To this end, the input dataset must be 
first encoded into a graph that can be directly processed by the sum-GNN, and the graph resulting from the sum-GNN application must be subsequently decoded back into an output dataset. Several
encoding/decoding schemes have been proposed in the literature. 
We adopt the so-called \emph{canonical scheme}, which is a straightforward way of converting datasets to coloured graphs. In particular, colours in graphs correspond  to binary predicates in the signature and channels of feature vectors in the input and output layers of the sum-GNN  to unary predicates. For each $p \in \{1, \dots, \delta\}$,
we  denote the unary predicate corresponding to channel $p$ as $U_p$.
More precisely, the canonical encoding $\texttt{enc}(D)$ of a dataset $D$
is the Boolean $(\text{$\col$},\delta)$-graph with a vertex $v_a$ for each constant $a$ in $D$ and a $c$-coloured edge $(v_a,v_b)$ for each
fact $R^c(a,b) \in D$. Furthermore, given a vertex $v_a$ corresponding to constant $a$, vector component $\mathbf{v}_a[p]$ is set to $1$ if and only if $U_p(a) \in D$, for $p \in \{1, \dots, \delta\}$.
The decoder $\texttt{dec}$ is the inverse of the encoder.
The \emph{canonical} dataset transformation induced by a sum-GNN $\mathcal{N}$ is then defined as: 
$T_{\mathcal{N}}(D) = \texttt{dec}(\mathcal{N}(\texttt{enc}(D))) .$
We abbreviate $\mathcal{N}(\texttt{enc}(D))$ by $\mathcal{N}(D)$.

\smallskip
\noindent
\textbf{Soundness and Completeness}
A Datalog program or rule $\alpha$ is sound for a 
sum-GNN $\mathcal{N}$ if $T_\alpha(D) \subseteq T_{\mathcal{N}}(D)$ for each dataset $D$.
Conversely, $\alpha$ is complete for $\mathcal{N}$ if $T_{\mathcal{N}}(D) \subseteq T_\alpha(D)$ for each dataset $D$.
Finally, we say that $\alpha$ is equivalent to $\mathcal{N}$ if it is both sound and complete for $\mathcal{N}$.
\changemarker{
The following proposition establishes that the containment relation that defines soundness for a rule or program still holds when the operators are composed a finite number of times.
}

\begin{proposition} \label{prop:sound_finite_containment}
\changemarker{
If $\alpha$ is a rule or program sound for sum-GNN $\mathcal{N}$, then for any dataset $D$ and $k \in \mathbb{N}$, the containment holds when $T_\alpha$ and $T_\mathcal{N}$ are composed $k$ times:
$T_\alpha^k(D) \subseteq T_{\mathcal{N}}^k(D)$.
}
\end{proposition}

\smallskip
\noindent
\textbf{Link Prediction} 
The link prediction task assumes a given incomplete dataset $D$ and an (unknown) completion $D^*$ of $D$ containing all missing \emph{binary} facts that are considered true 
over the predicates and constants of $D$.
Thus, given a fact $\alpha$ involving a binary predicate and constants from $D$, the task is to predict whether $\alpha \in D^*$.
\changemarker{To perform link prediction with a sum-GNN $\mathcal{N}$, we use the (non-canonical) encoding and decoding from 
\cite{cucala2021explainable},
which differs from the canonical encoding in that
vertices can now also encode pairs of constants, so that both unary and binary facts can be
represented in the channels of the input and output layers. This allows us to derive new binary facts (which the canonical transformation cannot do).
As shown in Section 3.2 of \cite{cucala2023correspondence},
this non-canonical encoding can be expressed as the composition of a Datalog program
$\mathcal{P}_\texttt{enc}$ and the canonical encoding; similarly, the decoding can be seen as the composition of the canonical decoding followed by the application of a Datalog program $\mathcal{P}_\texttt{dec}$. Hence, we  can use these programs to lift any rule extraction results obtained for the canonical transformation $T_{\mathcal{N}}$ to the end-to-end transformation 
$\mathcal{P}_\texttt{dec}(T_{\mathcal{N}}(\mathcal{P}_\texttt{enc}(D)))$.
}

\section{Partitioning the Channels of a GNN} \label{sec:sound}

In this section, we first provide two approaches for identifying channels in a sum-GNN that exhibit monotonic behaviour, \changemarker{which will later allow us to extract sound Datalog rules with head predicates corresponding to these channels.}
Candidate Datalog rules 
can be effectively checked for soundness using the approach developed by \cite{cucala2023correspondence} for monotonic GNNs.
Furthermore, we provide an approach for identifying channels for which no sound Datalog rule exists. 
Our techniques are data-independent and rely on direct analysis of the dependencies between feature vector components via the parameters of the model.

\subsection{Safe Channels}\label{sec:safe}

In this section, we introduce the notion of a safe channel of a sum-GNN $\mathcal{N}$. Intuitively, a channel is \emph{safe} if, for any dataset $D$, the computation of its value for each vertex $v \in \texttt{enc}(D)$
through the application of $\mathcal{N}$ to $D$ is affected only by non-negative values in the weight matrices of the GNN.

For instance, consider a simple sum-GNN where matrix  $\mathbf{A}_{\ell}$ for layer $\ell$ is given below 
$$
\mathbf{A}_{\ell} = \begin{pmatrix}
1 & 0 & 1 & 4 \\
1 & 0 & 0 & 2 \\
-8 & -1 & 0 & -2 \\
\end{pmatrix}
$$
and each of the matrices $\mathbf{B}_{\ell}^c$ for $c \in \col$ contain only zeroes. Furthermore, assume that \changemarker{layer $\ell-1$ has four channels, where the \emph{third} one has been identified as unsafe and all other channels have been identified as safe. Then, for an arbitrary vertex $v$, the product of $\mathbf{A}_{\ell}$ and  $\mathbf{v}_{\ell-1}$ in the computation of $\mathbf{v}_{\ell}$
reveals that the second channel in $\ell$ is safe, because
the unsafe component in $\mathbf{v}_{\ell-1}$ is multiplied by zero ($\mathbf{A}_\ell[2,3]$),
and the safe components are multiplied by non-negative matrix values.
In contrast, the first channel of $\ell$ is unsafe since its computation involves the product of an unsafe component of $\mathbf{v}_{\ell-1}$ with a non-zero matrix component ($\mathbf{A}_\ell[1,3]$), and the third channel is unsafe due to the product of a component of $\mathbf{v}_{\ell-1}$ with a negative matrix value (e.g.\ $\mathbf{A}_\ell[3,1]$).}
The following definition generalises this intuition.

\begin{definition} \label{def:nabn}
Let $\mathcal{N}$ be a sum-GNN as in \Cref{eq:GNN}.
All channels 
$i \in \{ 1, ..., \delta_0 \}$ are \emph{safe}
at layer $\ell =0$. Channel $i \in \{1, \ldots, \delta_{\ell}\}$ is \emph{safe} at layer  $\ell \in \{1, \ldots, L\}$ if the $i$-th row of each matrix $\mathbf{A}_\ell$ and $\{(\mathbf{B}_\ell^c)\}_{c \in \text{$\col$}}$ contains only non-negative values and, additionally, the $j$-th element in each such row is zero for each $j \in \{1, \dots, \delta_{\ell-1}\}$ such that  $j$ is unsafe in layer $\ell-1$. Otherwise, $i$ is \emph{unsafe}.
\end{definition}

We can now show that safe channels in a GNN exhibit monotonic behaviour. In particular, the value of a safe channel may only increase (or stay the same) when new facts are added to the input dataset.

\begin{lemma} \label{lemma:nabn_monotonic}
Let $\mathcal{N}$ be a sum-GNN as in \Cref{eq:GNN}, 
let $D', D$ be datasets satisfying $D' \subseteq D$, let $v \in \texttt{enc}(D')$, and let  $\mathbf{v}_\ell$ and  $\mathbf{v}_\ell'$ be the vectors associated to $v$ in layer $\ell$ upon applying $\mathcal{N}$ to $D$ and $D'$ respectively.
If channel $i$ in $\mathcal{N}$ is safe at layer $\ell$, then $\mathbf{v}_\ell'[i] \leq \mathbf{v}_\ell[i]$.
\end{lemma}

\begin{proof}
We proceed by induction on $\ell$. The base case $\ell =0$
holds trivially by the definition of the canonical encoding and the fact that $D' \subseteq D$.
For the inductive step, assume that the claim holds for $\ell -1\geq 0$ and that channel $i$ is safe at layer $\ell$. We show that $\mathbf{v}_\ell'[i] \leq \mathbf{v}_\ell[i]$.
Consider the computation of $\mathbf{v}_\ell[i]$ and $\mathbf{v}_\ell'[i]$  as per  \Cref{align:def_gnn_update}. The value $(\mathbf{A}_\ell \mathbf{v}_{\ell-1})[i]$ is the sum over all $j \in \{ 1, ..., \delta_{\ell-1} \}$ of $\mathbf{A}_\ell[i, j] \mathbf{v}_{\ell-1}[j]$. By \Cref{def:nabn}, this sum involves only non-negative values and can be restricted to values of $j$ corresponding to safe channels at layer $\ell-1$.
By induction, each such safe $j$ satisfies $\mathbf{v}_{\ell-1}'[j] \leq \mathbf{v}_{\ell-1}[j]$ and hence  $\mathbf{A}_\ell[i, j] \mathbf{v}_{\ell-1}'[j] \leq \mathbf{A}_\ell[i, j] \mathbf{v}_{\ell-1}[j]$; thus, $(\mathbf{A}_\ell \mathbf{v}_{\ell-1}')[i] \leq (\mathbf{A}_\ell \mathbf{v}_{\ell-1})[i].$

Consider now  $c \in \text{$\col$}$ and let $E^c$ and $(E^c)'$ be the $c$-coloured edges in $\texttt{enc}(D)$ and  $\texttt{enc}(D')$ respectively. The sums involved in the products $(\mathbf{B}_\ell^c \sum_{(v,u) \in E^c} \mathbf{u}_{\ell-1} )[i]$ and $(\mathbf{B}_\ell^c \sum_{(v,u) \in (E^c)'} \mathbf{u}_{\ell-1}' )[i]$ contain only non-negative values and
can similarly be  restricted to safe channels at layer $\ell-1$. By induction, each such safe $j$ satisfies that $\forall u \in \texttt{enc}(D'),~ \mathbf{u}_{\ell-1}'[j] \leq \mathbf{u}_{\ell-1}[j]$. Furthermore, $(E^c)' \subseteq E^c$ given that $D' \subseteq D$. We conclude that
$$ (\sum_{c \in \text{$\col$}} \mathbf{B}_\ell^c \sum_{(v,u) \in (E^c)'} \mathbf{u}'_{\ell-1} )[i] \leq (\sum_{c \in \text{$\col$}} \mathbf{B}_\ell^c \sum_{(v,u) \in E^c} \mathbf{u}_{\ell-1} )[i] . $$


By combining the previous inequalities and taking into account that $\sigma_\ell$ is monotonically increasing, we conclude that $\mathbf{v}_\ell'[i] \leq \mathbf{v}_\ell[i]$, as required.

\end{proof}

We conclude this section by showing that the identification of safe channels \changemarker{allows for} the extraction of sound rules from a trained sum-GNN model. In particular, given a candidate Datalog rule, it is possible to algorithmically verify its soundness using the approach of \cite{cucala2023correspondence}. The soundness check for a rule $r$ of the form \eqref{eq:ruleform} involves considering an arbitrary (but fixed) set containing as many constants as variables in $r$, and considering each  substitution $\nu$ mapping variables in $r$ to these constants and satisfying the inequalities in $r$. For $D_r^{\nu}$, the dataset consisting of each fact $B_i\nu$ such that $B_i$ is a body atom in $r$, we check whether the GNN predicts the fact $H\nu$, corresponding to the grounding of the head atom. If this holds for each considered $\nu$, then the rule is sound; otherwise, the substitution $\nu$ for which it does not hold provides a counter-example to soundness.

\begin{proposition} \label{prop:nabn_sound}
Let $\mathcal{N}$ be a sum-GNN as in \Cref{eq:GNN}, and let $r$  be a rule of the form \eqref{eq:ruleform} where $H$ mentions a unary predicate $U_p$. Let $S$ be an arbitrary set of as many constants as there are variables in $r$.
Assume channel $p$ in $\mathcal{N}$ is safe at layer $L$.
Then $r$ is sound for $T_\mathcal{N}$ if and only if
$H\nu \in T_\mathcal{N}(D_r^{\nu})$
for each substitution $\nu$ mapping the variables of $r$ to constants in $S$ and such that $D_r^{\nu} \models B_i\nu$ for each inequality $B_i$ in r.
\end{proposition}

\begin{proof}
To prove the soundness of $r$, consider an arbitrary dataset $D$. We show that $T_r(D) \subseteq T_\mathcal{N}(D)$. To this end, we consider an arbitrary fact in $T_r(D)$ and show that it is also contained in 
$T_\mathcal{N}(D)$. By the definition of the immediate consequence operator $T_r$, this fact is of the form $H\mu$, where $\mu$ is a substitution from the variables of $r$ to constants 
in $D$ satisfying $D \models B_i\mu$ for each body literal $B_i$ of $r$. Let 
$\sigma$ be an arbitrary one-to-one mapping from the co-domain of $\mu$ to some subset of $S$; such a mapping exists because $S$ has as least as many constants as variables in $r$. Let $\nu$ be the composition of $\mu$ and $\sigma$.

Observe that for each body inequality $B_i$ of $r$, we have $D_r^{\nu} \models B_i \nu$
because $D \models B_i \mu$
and $\sigma$ is injective. 
Therefore, by the assumption of the proposition, $H\nu \in T_{\mathcal{N}}(D_r^{\nu})$.
Now, observe that the result of applying $T_{\mathcal{N}}$ to a dataset does not depend on the identity of the constants, but only on the structure of the dataset; therefore, $T_{\mathcal{N}}$ is invariant under one-to-one mappings of constants, and since $\sigma$ is one such map, $H\nu \in T_{\mathcal{N}}(D_r^{\nu})$ implies
$H\mu \in T_{\mathcal{N}}(D_r^{\mu})$.
Now, let $a$ be the single constant in $H \mu$.
Since $D^r_{\mu} \subseteq D$ by definition of $\mu$, and channel $p$ is safe at layer $L$, 
we can apply \Cref{lemma:nabn_monotonic} to conclude that $\mathbf{v}_L'[p] \leq \mathbf{v}_L[p]$, for $v$ the vertex corresponding to $a$ in $\texttt{enc}(D_r^{\mu})$, and $\mathbf{v}'_L$ and $\mathbf{v}_L$ the feature vectors in layer $L$ computed for $v$ by $\mathcal{N}$ on $D^{\mu}_r$ and $D$, respectively. But $H \mu \in T_\mathcal{N}(D^\mu_r)$ implies that $\texttt{cls}_t(\mathbf{v}_L'[p])=1$ and so 
$\mathbf{v}_L'[p] \geq t$.
Hence, $\mathbf{v}_L[p] \geq t$, and so
$\texttt{cls}_t(\mathbf{v}_L[p])=1$,
which implies that $H \mu \in T_\mathcal{N}(D)$, as
we wanted to show.

If on the other hand, if $H\nu \not \in T_\mathcal{N}(D_r^\nu)$ for some substitution $\nu$ defined as in the proposition, then $T_r(D_r^\nu) \not\subseteq T_\mathcal{N}(D_r^\nu)$, as $H\nu \in T_r(D_r^\nu)$. Thus, $r$ is unsound for $T_\mathcal{N}$.
\end{proof}

The identification of safe channels provides a sufficient condition for monotonic behaviour \changemarker{that can be easily computed in practice} and enables rule extraction. The fact that a channel is unsafe, however, does not imply that it behaves non-monotonically.
\changemarker{In the following section, we provide a more involved analysis of the dependencies between channels and model parameters which yields a more fine-grained channel classification. In particular, we identify two new classes of channels that behave monotonically, such that their union contains all safe channels, but may also contain unsafe channels.}

\subsection{Stable and Increasing Channels}\label{sec:stable}

In this section, we provide a classification of the channels of the GNN depending on their behaviour under updates involving the addition of new facts to an input dataset. 
Intuitively,  \emph{stable} channels are those whose value always remains unaffected by such updates; in turn, \emph{increasing} channels cannot decrease in value, whereas \emph{decreasing} channels cannot increase in value. All remaining channels are categorised as \emph{undetermined}.
Consider again a sum-GNN where matrix  $\mathbf{A}_{\ell}$ for layer $\ell$ is given below and each of the matrices $\mathbf{B}_{\ell}^c$ for $c \in \col$ contain only zeroes.
$$
\mathbf{A}_\ell = \begin{pmatrix}
1 & 0 & -1 & -2 \\
2 & 1 & -3 & 0 \\
1 & 0 & 1 & 2 \\
-3 & 0 & 2 & 0 \\
\end{pmatrix}
$$
 \changemarker{Furthermore, assume again that layer $\ell-1$
has four channels, where the first has been identified as increasing, the second one as undetermined, the third one as decreasing, and the fourth one as stable.
For an arbitrary vertex $v$,
the product of $\mathbf{A}_{\ell}$ and $\mathbf{v}_{\ell-1}$
in the computation of $\mathbf{v}_{\ell}$ 
reveals that the first channel of $\ell$ is increasing 
since the undetermined component of $\mathbf{v}_{\ell-1}$ is multiplied by $0$, the increasing component by a positive matrix value, and the decreasing component by a negative matrix value}. The second channel is undetermined, since it involves the product of an undetermined component with a non-zero matrix value \changemarker{($\mathbf{A}_\ell[1,1]$)}.
The third channel is also undetermined, since it is a mixture of increasing and decreasing values: positive-times-increasing is increasing, whereas positive-times-decreasing is decreasing, so it cannot be known a-priori whether the sum of these values will increase or decrease.
Finally, the fourth channel of $\ell$ is decreasing 
since the undetermined component of $\mathbf{v}_{\ell-1}$ is multiplied by $0$, and increasing (resp.\ decreasing) components are multiplied by negative (resp.\ positive) matrix values.

The following definition formalises these intuitions and extends the analysis to the matrices $\mathbf{B}_{\ell}^c$ involved in neighbourhood aggregation.

\begin{definition} \label{def:updown}
Let $\mathcal{N}$ be a sum-GNN as in Equation \eqref{eq:GNN}. 
All channels are \emph{increasing} at layer $0$, 
A channel $i \in \{1, \dots, \delta_\ell\}$ is \emph{stable} at layer $\ell \in \{1, ..., L\}$ if both of the following conditions hold for each $j \in \{1, \dots, \delta_{\ell-1}\}$:
\begin{itemize}
 \item $\mathbf{B}^c_{\ell}[i,j] = 0$ for each  $c \in \col$, and
 \item $\mathbf{A}_{\ell}[i,j] \neq 0$ implies that $j$ is stable in layer $\ell-1$.
\end{itemize}
It is \emph{increasing} (resp.\ \emph{decreasing}) at layer $\ell$ if it is not stable and, for each $j \in \{1, \dots, \delta_{\ell-1}\}$, these conditions hold:
\begin{enumerate}
        \item if $j$ is increasing in $\ell-1$, 
        then $\mathbf{A}_{\ell}[i,j] \geq 0$ (resp.\ $\mathbf{A}_{\ell}[i,j] \leq 0$);
        \item if $j$ is decreasing in $\ell-1$, then $\mathbf{A}_{\ell}[i,j] \leq 0 $  (resp.\ $\mathbf{A}_{\ell}[i,j] \geq 0 $) and $\mathbf{B}^c_{\ell}[i,j] = 0$ for each $c \in \col$; 
        \item if $j$ is undetermined in $\ell-1$, then $\mathbf{A}_{\ell}[i,j] = 0$ and
        ${\mathbf{B}^c_{\ell}[i,j] = 0}$ for each $c \in \col$; and
        \item $\mathbf{B}^c_{\ell}[i,j] \geq 0$ (resp.\ $\mathbf{B}^c_{\ell}[i,j] \leq 0$) for each $c \in \col$.
\end{enumerate}
It is \emph{undetermined} at layer $\ell$ if it is neither stable, increasing, nor decreasing. 
\end{definition}

Note that a channel cannot be both
increasing and decreasing, because
satisfying the conditions for both classes would imply that it is stable, 
which is incompatible with being increasing or decreasing. 

The following lemma shows that the behaviour of the channel types aligns with their intended interpretation.

\begin{lemma} \label{lemma:updown_monotonic}
Let $\mathcal{N}$ be a sum-GNN of $L$ layers, and let $D', D$ be datasets satisfying $D' \subseteq D$.
For each vertex $v \in \texttt{enc}(D')$, layer $\ell \in \{0, \dots, L\}$, and channel $i \in \{1, \dots, \delta_{\ell}\}$,
the following hold:
\begin{itemize}
\item If $i$ is stable at layer $\ell$, then $\mathbf{v}_\ell'[i] = \mathbf{v}_\ell[i]$;
\item If $i$ is increasing at layer $\ell$, then $\mathbf{v}_\ell'[i] \leq \mathbf{v}_\ell[i]$, and 
\item If $i$ is decreasing at layer $\ell$, then $\mathbf{v}_\ell'[i] \geq \mathbf{v}_\ell[i]$,
\end{itemize}
where $\mathbf{v}_\ell$ and $\mathbf{v}_\ell'$ are the vectors induced for $v$ in layer $\ell$ by applying $\mathcal{N}$ to $D$ and $D'$ respectively.
\end{lemma}

\begin{proof}[Proof sketch]
The full proof 
is given in \Cref{app:updown_monotonic_proof}.

We show the claim of the lemma by induction on $\ell$. The base case holds because 
$\mathbf{v}'_0[i] \leq \mathbf{v}_0[i]$ for each $i \in \{1, \dots, \delta\}$ by definition of $\texttt{enc}$,
and all channels are increasing in layer $0$. For the inductive step, we prove that it holds at layer $\ell$ for stable, increasing, and decreasing $i$.

For stable $i$, notice that $(\mathbf{A}_\ell \mathbf{v}_{\ell-1})[i]$ is just a sum over channels $j$ that are stable at $\ell-1$, since the non-stable $j$'s are zeroed out. The induction hypothesis then implies $\mathbf{v}_{\ell-1}'[j] = \mathbf{v}_{\ell-1}[j]$ for each such $j$. Furthermore, $\mathbf{B}_\ell^c[i, j] = 0$ for every $j$ and $c$. Hence, $\mathbf{v}_\ell'[i] = \mathbf{v}_\ell[i]$.

For increasing $i$, consider the four possibilities for $j$ at $\ell-1$ and conditions (1) - (3) in the definitions. For example, if $j$ is increasing, then from (1) we have $\mathbf{A}_\ell[i,j] \geq 0$ and by our induction hypothesis, $\mathbf{v}_{\ell-1}'[j] \leq \mathbf{v}_{\ell-1}[j]$. In any of these cases, we find that $\mathbf{A}_\ell[i, j] \mathbf{v}_{\ell-1}'[j] \leq \mathbf{A}_\ell[i, j] \mathbf{v}_{\ell-1}[j]$.
For the product involving $\mathbf{B}_\ell^c$, if $j$ is decreasing or undetermined at $\ell-1$ then from (2) and (3) we have that $\mathbf{B}_\ell^c[i, j] = 0$. Otherwise, by the inductive hypothesis we obtain $\mathbf{u}'_{\ell-1}[j] \leq \mathbf{u}_{\ell-1}[j]$ for every node $u \in \texttt{enc}(D')$. Then since $(E^c)' \subseteq E^c$, the inequality is preserved when summing over neighbours of $v$. Also, since from (4) we have that $\mathbf{B}_\ell^c[i, j] \geq 0$, we find that for all $j$,
$$ \mathbf{B}_\ell^c[i, j] ( \sum_{(v,u) \in (E^c)'} \mathbf{u}'_{\ell-1} )[j] \leq \mathbf{B}_\ell^c[i, j] ( \sum_{(v,u) \in E^c} \mathbf{u}_{\ell-1} )[j] . $$
Therefore, by monotonicity of $\sigma$, $\mathbf{v}_\ell'[i] \leq \mathbf{v}_\ell[i]$.
For decreasing $i$, the proof is very similar to that for increasing $i$.
\end{proof}
Both increasing and stable channels are amenable to rule extraction. In particular, the soundness check in Proposition \ref{prop:nabn_sound} extends seamlessly to this new setting. 

\begin{proposition} \label{prop:updown_sound}
Let $\mathcal{N}$ be a sum-GNN as in \Cref{eq:GNN}, and let $r$  be a rule of the form \eqref{eq:ruleform} where $H$ mentions a unary predicate $U_p$. Let $S$ be an arbitrary set of as many constants as there are variables in $r$.
Assume channel $p$ in $\mathcal{N}$ is increasing or stable at layer $L$.
Then $r$ is sound for $T_\mathcal{N}$ if and only if
$H\nu \in T_\mathcal{N}(D_r^{\nu})$
for each substitution $\nu$ mapping the variables of $r$ to constants in $S$ and such that $D_r^{\nu} \models B_i\nu$ for each inequality $B_i$ in r.
\end{proposition}

\changemarker{
On the other hand, if a channel is decreasing or undetermined, it does not behave monotonically, and so it may or may not have sound rules.}
\changemarker{We conclude this section
by relating the classes of channels described here to those of Section \ref{sec:safe}, with a full proof given in \Cref{app:channel_relations_proof}.
}

\begin{theorem} \label{thm:channel_relations}
\changemarker{
Safe channels are increasing or stable.
There exist increasing unsafe channels and stable unsafe channels.
}
\end{theorem}
\subsection{Unbounded Channels} \label{sec:unsound}

In Section \ref{sec:stable}, we noted that extracting sound Datalog rules for channels that are decreasing or undetermined might be possible.
In this section, we identify a subset of such channels, which we refer to as \emph{unbounded}, for which \emph{no} sound Datalog rule may exist.  Thus, being unbounded provides a sufficient condition for non-monotonic channel behaviour.

The techniques in this section require that the sum-GNN uses ReLU as the activation function in all but (possibly) the last layer $L$.
Furthermore, we require the co-domain of the activation function in layer $L$ to include a number strictly less than the threshold $t$ of the classification function $\texttt{cls}_t$.
This restriction is non-essential and simply excludes GNNs that derive all possible facts regardless of the input dataset,
in which case
all rules would be sound.

\begin{definition} \label{def:neginfline}
Channel $p \in\{1, \dots, \delta_L\}$ is \emph{unbounded} at layer $L$ 
if there exist a Boolean vector $\mathbf{y}_0$ of dimension $\delta_0$, and 
  a sequence $c_1, \dots, c_L$ of (not necessarily distinct) colours in $\col$,
such that, with $\{\mathbf{y}_\ell\}_{\ell=1}^{L-1}$ the sequence defined inductively 
as $\mathbf{y}_\ell := \text{ReLU} ( \mathbf{B}_\ell^{c_\ell} \mathbf{y}_{\ell - 1} )$ for each $1 \leq \ell \leq L-1$, it holds that
$(\mathbf{B}_L^{c_L} \mathbf{y}_{L-1}) [p] <0$.
\end{definition}

Intuitively, the value of an unbounded channel at layer $L$ for a given vertex
can always be made smaller
than the classification threshold $t$ by extending
the
input graph in a precise way. This  helps us prove that 
each rule $r$ with head predicate $U_p$ corresponding to an unbounded channel $p$ is not sound for $T_\mathcal{N}$.
In particular,
we first generate the dataset $D^\nu_r$, where $\nu$ is a substitution that maps each variable in $r$ to a different constant, and we let $a := \nu(x)$.
Then we extend $D^{\nu}_r$ to a dataset $D'$ 
in a way that ensures
that the value of channel $p$ at layer $L$ for the vertex $v_a$ in $D'$ is smaller than the threshold
$t$.
\changemarker{Thus}, $U_p(a) \notin T_{\mathcal{N}}(D')$, even though $U_p(a)$ is clearly in $T_{r}(D')$, \changemarker{so dataset $D'$ is a counterexample to the soundness of $r$}.

The following theorem formally states this result. The full proof of the theorem is given in \Cref{app:neginfline_sound_proof}.

\begin{theorem} \label{thm:neginfline_sound}
Let $\mathcal{N}$ be a sum-GNN as in Equation \eqref{eq:GNN}, where 
$\sigma_{\ell}$ is ReLU for each $1 \leq \ell \leq L-1$,
and the co-domain of $\sigma_L$ contains a number strictly less than the threshold $t$ of the classification function $\texttt{cls}_t$.
Then, each rule with head predicate $U_p$ corresponding to an unbounded channel $p$ is unsound for $\mathcal{N}$.
\end{theorem}

\begin{proof}[Proof sketch]
Consider an unbounded channel $p$ at layer $L$ and a rule $r$ with head $U_p(x)$.
Let $\nu$ be an arbitrary substitution mapping each variable in $r$ to a 
different constant.
Let $G= \texttt{enc}( D_r^{\nu})$ and let $v_a$ be the vertex
corresponding to constant $a:= \nu(x)$.
Let 
$\mathbf{y}_0$ and $c_1, \dots, c_L$ be a Boolean vector and a sequence of colours, 
respectively, satisfying the condition in Definition \ref{def:neginfline} 
for $p$, and let  
$\{\mathbf{y}_\ell\}_{\ell=1}^{L-1}$
be the sequence computed for them as in the definition.

For each $d \in \mathbb{N}$, let $D_d$ be the extension of $D_r^{\nu}$ with
\begin{enumerate}[leftmargin=1cm]
    \item facts $R^{c_1}(a_j,b_1)$ for $1 \leq j \leq d$, facts $R^{c_{\ell}}(b_{\ell-1},b_{\ell})$ for $2 \leq \ell \leq L - 1$, and fact $R^{c_L}(b_{L-1},a)$;
    \item facts $U_k(a_j)$ for $1 \leq j \leq d$ and $k \leq \delta$ s.t.\  
    $\mathbf{y}_0[k]=1$;
\end{enumerate}
for $a_j$ and $b_\ell$ distinct constants not in $D^{\nu}_r$ for each $j$ and $\ell$. Furthermore, let $G_d$ be the canonical encoding of $D_d$,
where $v_\ell$ is the vertex corresponding to constant $b_\ell$
and $u_j$ is the vertex corresponding to constant $a_j$. 
The resulting graph $G_d$ is illustrated in Figure \ref{fig:counterexample},
where edges labelled with question marks represent the
canonical encoding of $D^\nu_r$.

\begin{figure}
\begin{tikzpicture}
    \node[shape=circle,draw=black, minimum size =1cm] (u1) at (0, 11) {$u_1$};
    \node[shape=circle,draw=black, minimum size =1cm] (u) at (-1.5, 10.5) {$...$};
    \node[shape=circle,draw=black, minimum size =1cm] (ud) at (-1.5, 9) {$u_d$};
    \node[shape=circle,draw=black, minimum size =1cm] (v1) at (0, 9) {$v_1$};
    \node[shape=circle,draw=black, minimum size =1cm] (v2) at (1.5, 9) {$v_2$};
    \node[shape=circle,draw=black, minimum size =1cm] (v) at (3.5, 9) {$...$};
    \node[shape=circle,draw=black,scale=0.9] (vL) at (5.5, 9) {$v_{L-1}$};
    \node[shape=circle,draw=black, minimum size =1cm] (vx) at (3.5, 11) {$v_{a}$};
    \node[shape=circle,draw=black, minimum size =1cm] (b1) at (5.5, 11) {$...$};
    \node[shape=circle,draw=black, minimum size =1cm] (b2) at (1.5, 11) {$...$};

    \path [->] (u1) edge node[left] {${c_1}$} (v1);
    \path [->] (u) edge node[below] {${c_1}$} (v1);
    \path [->] (ud) edge node[below] {${\;\;c_1}$} (v1);
    \path [->] (v1) edge node[below] {${c_2}$} (v2);
    \path [->] (v2) edge node[below] {${c_3}$} (v);
    \path [->] (v) edge node[below] {${c_{L-1}}$} (vL);
    \path [->] (vL) edge node[right] {${\;c_{L}}$} (vx);
    \path [<-] (b1) edge node[above] {$?$} (vx);
    \path [->] (b2) edge node[above] {$?$} (vx);
\end{tikzpicture}
\caption{Canonical encoding $G_d$ of the dataset $D_d$ that extends $D^{\nu}_r$ in the proof of Theorem \ref{thm:neginfline_sound} for a given $d \in \mathbb{N}$.}
\label{fig:counterexample}
\end{figure}
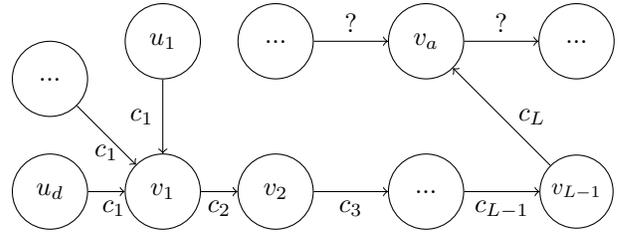

We next show that there exists $d^* \in \mathbb{N}$ such that 
the value of channel $p$ at layer $L$ for vertex $v_a$ in the canonical encoding of $D_{d^*}$ 
is under the threshold $t$.
To find $d^*$, we inductively define a number $d_{\ell}$
for each $1 \leq \ell \leq L-1$
that satisfies the following condition: 
\begin{enumerate}[label=(C\arabic*),leftmargin=1cm]
\item for all $d \geq d_{\ell}$, the value of the feature vector computed by $\mathcal{N}$ on $G_d$ in layer $\ell$ for vertex
$v_{\ell}$ is $\mathbf{g}_{\ell} + d \cdot \mathbf{y}_{\ell}$,
where $\mathbf{g}_{\ell}$ is a fixed vector that does not depend on $d$.\label{cond:linearform}
\end{enumerate}
Each $d_\ell$ can be defined inductively (assuming $d_0=1$) as follows.
\changemarker{
Condition \ref{cond:linearform} ensures that if $d \geq d_{\ell-1}$, the value of 
$\mathbf{v}_{\ell}[j]$ computed by equation
\eqref{align:def_gnn_update} 
for vertex $v=v_{\ell}$ in layer $\ell$ when applying $\mathcal{N}$ to graph $G_d$ is of the form $\mbox{ReLU}(S_{j,d})$, for $S_{j,d}$ a sum that contains a term of the form 
$z_j = d \cdot (\mathbf{B}^{c_\ell}_\ell \mathbf{y}_{\ell-1}) [j]$
and no other terms depending on $d$.
We define $d_\ell$ as the smallest integer large enough to ensure that each $z_j$ dominates the sum of all other terms in $S_{j,d}$.
Hence, if $d \geq d_{\ell}$, we have that if $z_j$ is positive, then so is $S_{j,d}$, and ReLU will act on it as the identity;
if $z_j$ is negative, then so is $S_{j,d}$,
and ReLU will map it to $0$, and if $z_j =0$, then $\mbox{ReLU}(S_{j,d})$
will be independent of $d$. Thus, in all three cases, $\mbox{ReLU}(S_{j,d})$ is 
 of the form $\mathbf{g}_\ell[j] + d \cdot \mathbf{y}_{\ell}[j]$, for $\mathbf{g}_\ell[j]$ a number that does not depend on $d$,
 and so Condition \ref{cond:linearform} holds for $d_{\ell}$ as required.}

\changemarker{Finally, we consider the instance of equation
\eqref{align:def_gnn_update} 
that computes the value of vertex $v_a$ in layer $L$ when applying $\mathcal{N}$ to graph $G_d$, for $d \geq d_{\ell-1}$.
Condition \ref{cond:linearform} ensures that the value of $\mathbf{v}_L[p]$ is the result of applying $\sigma_{L}$ 
to a sum that contains a term of the form 
$d \cdot (\mathbf{B}^{c_L}_L \mathbf{y}_{L-1}) [p]$ and no other terms depending on $d$.} 
The definition of an unbounded channel then ensures
that $(\mathbf{B}^{c_L}_L \mathbf{y}_{L-1}) [p]$ is negative,
and so by choosing $d$ sufficiently large, 
the aforementioned sum is arbitrarily small towards $-\infty$.
But $\sigma_L$ is monotonic and includes \changemarker{a number less than} $t$ in its codomain, which means that 
there exists a minimum natural number $d$ such that the image by $\sigma_L$
of the aforementioned sum is under the threshold $t$; we then let 
$d^*$ be this number.

Dataset $D_{d^*}$ is therefore a counterexample showing that $r$ is not sound
for $\mathcal{N}$. Indeed, when applying $\mathcal{N}$ to $G_{d^*}$, the value of channel $p$ for vertex $v_a$ in layer $L$
is smaller than the threshold 
$t$, and so we have that $U_p(a) \notin T_{\mathcal{N}}(D_{d^*})$.
However, $U_p(a) \in T_r(D_{d^*})$, since $D_{d^*}$ contains $D_r^{\nu}$ and 
all inequalities in the body of $r$ are satisfied because $\nu$ maps each variable to a different
constant. Therefore, rule $r$ is unsound for $\mathcal{N}$. 
\end{proof}

In order to check which channels are unbounded in practice, we initialise an
empty set $S$ of unbounded channels. We then enumerate all Boolean feature
vectors $\mathbf{y}_0$ of dimension $\delta_0$ and all sequences $c_1, \dots, c_L$
of colours; for each vector and sequence, we compute the vector $\mathbf{B}^{c_L}_{L} \mathbf{y}_{L-1}$ as shown in Definition \ref{def:neginfline} and add to $S$ each
channel $p$ such that 
$(\mathbf{B}^{c_L}_{L} \mathbf{y}_{L-1})[p]$ is negative. 
The enumeration can be stopped early if all channels are found
to be unbounded (i.e.\ whenever $S$ becomes $\{1, \dots, \delta\}$). 

In most practical cases, however, the dimension $\delta$ and number
of colours $|\col|$ and layers $L$ are large enough that it is not practically feasible to enumerate all combinations.
In those cases, we randomly sample a fixed number of pairs
$(\mathbf{y}_0,\{c_\ell\}_{\ell=1}^L)$ of Boolean feature vectors and
sequences of colours. 
All channels in $S$ after this iteration
will be unbounded, but channels not in $S$ may also be unbounded. 

\changemarker{

To conclude this section, we relate unbounded channels to those defined in \Cref{sec:safe} and \Cref{sec:stable}.
A full proof is given in \Cref{app:unbounded_channel_relations_proof}.

}

\begin{theorem} \label{thm:unbounded_channel_relations}
\changemarker{
Unbounded channels are neither increasing, nor stable, nor safe. There exist, however, decreasing unbounded channels and undetermined unbounded channels.
}
\end{theorem}
\section{Experiments} \label{sec:experiments}
We train sum-GNNs on several link prediction
datasets, using the dataset transformation described in \Cref{sec:background}.
For each dataset, we train a sum-GNN with ReLU activation functions, biases, and two layers (this architecture corresponds to R-GCN with biases). The hidden layer of the GNN has twice as many channels as its input. Moreover, we also train 
four additional instances of sum-GNNs, using a modified 
training paradigm to facilitate the learning of 
sound rules (see \Cref{sec:training}).
For each trained model,
we compute standard classification metrics, such as  precision, recall,
accuracy, and F1 score, and 
area under the precision-recall curve (AUPRC).

We train all our models using binary cross entropy loss for training and the Adam optimizer with a standard learning rate of $0.001$. 
For each model, we choose the classification threshold by computing the accuracy on the
validation set across a range of $108$ thresholds between $0$ and $1$ and 
selecting the one which maximises accuracy.
We run each experiment across $10$ different random seeds and present the aggregated metrics. Experiments are run using PyTorch Geometric, with a CPU on a Linux server.

\paragraph{Channel Classification and Rule Extraction}
For each trained model,
we compute which output channels of the model were safe, stable, increasing, and unbounded (using $1000$ random samples of pairs of Boolean feature vectors and colour sequences).
On all datasets, for each channel $p$ that is stable or increasing, we iterate over each Datalog rule in the signature with up to two body atoms and a head predicate $U_p$, and count the number of sound rules, using \Cref{prop:updown_sound} to check soundness.
In benchmarks with a large number of predicates, we only check rules with one body atom, since searching the space of rules with two body atoms is intractable.
For datasets created with LogInfer \cite{liu2023revisiting}, 
which are obtained by enriching a pre-existing 
dataset with the consequences of a known set of Datalog rules, we also check if these rules were sound for the model.

\subsection{Datasets}

We use three benchmarks provided by \cite{teru2020inductive}: WN18RRv1, FB237v1, and NELLv1;
each of these benchmarks provides datasets for training, validation, and testing, as well as negative examples.

We also use LogInfer \cite{liu2023revisiting}, a framework which augments a dataset by applying Datalog
rules of a certain shape---called a ``pattern'''---and adding the consequences of the rules back to the dataset. We apply the LogInfer framework to 
datasets FB15K-237 \cite{toutanova2015observed} and WN18RR \cite{dettmers2018convolutional}.
We consider the very simple rule patterns \emph{hierarchy} and \emph{symmetry} defined in 
\cite{liu2023revisiting}; we also use a new monotonic pattern \emph{cup},
which has a tree-like structure, and a non-monotonic rule pattern \emph{non-monotonic hierarchy}  (nmhier); all patterns are shown in Table \ref{tab:loginfer_patterns}.
For each pattern P, we refer to the datasets obtained by enriching 
FB15K-237 and WN18RR by FB-P and WN-P, respectively.
For each dataset and pattern, we randomly select 10\% of the enriched training dataset to be used as targets and the rest as inputs to the model. Furthermore, we consider an additional instance of 
FB15K-237 and WN18RR where we again apply the hierarchy pattern, but
we include a much larger number of consequences in the enriched training dataset; 
we refer to these datasets as FB-superhier and WN-superhier, as a shorthand for \emph{super-hierarchy}.  
The purpose of this is to create a dataset where it should be as easy as possible for a model
to learn the rules applied in the dataset's creation. 
Finally, we use LogInfer to augment FB15K-237 and WN18RR using multiple rule patterns at 
the same time; in particular, we combine monotonic with non-monotonic rule patterns. 
This allows us to test the ability of our rule extraction methods to recover the sound monotonic rules that were used to generate the enriched dataset, in the presence of facts derived by non-monotonic rules. For example, WN-hier\_nmhier refers to a LogInfer dataset where WN is extended using the hierarchy and non-monotonic hierarchy rule patterns.
When creating a mixed dataset, we reserve a number of predicates to use in the heads of the one pattern, and the rest of the predicates to use in the heads of the other.

Negative examples are generated for LogInfer using predicate corruption: i.e. for each positive example $P(a, b) \in D$, we sample a predicate $Q$ at random to obtain a negative example $Q(a, b)$ such that it does not appear in the training, validation, or test set.
\changemarker{We avoid using constant corruption
to produce negative examples because this would inflate the performance of our models. The reason for this is that the encoding by \cite{cucala2021explainable} that we use in our transformation can only predict binary facts for constant pairs occurring together in input facts, and hence would trivially (and correctly) classify almost all negative examples. }
When generating negative examples for a hier\_nmhier dataset, we sample $T(a, b)$ such that $\{ R(a, b), S(b, c) \} \subseteq D$, to penalise the model for learning the rule pattern $R(x, y) \rightarrow T(x, y)$ instead of the non-monotonic one.
We train for $8000$ epochs on WN18RRv1, FB237v1, NELLv1, and the LogInfer-WN datasets, 
and for $1500$ epochs on the LogInfer-FB datasets.

\begin{table}
\centering
\resizebox{1\columnwidth}{!}{
\begin{tabular}{ll}
\toprule
 & Pattern \\
\midrule
Hierarchy (hier) & $R(x, y) \rightarrow S(x, y)$ \\
Symmetry (sym) & $R(x, y) \rightarrow R(y, x)$ \\

\midrule
Cup & $R(x, y) \land S(y, z) \land T(w, x) \rightarrow P(x, y)$ \\
NM-Hierarchy & $R(x, y) \land \neg S(y, z) \rightarrow T(x, y)$ \\

\bottomrule
\end{tabular}
} 
\caption{LogInfer inference patterns used in this paper.}
\label{tab:loginfer_patterns}
\end{table}

\subsection{Training Paradigm Variations}\label{sec:training}
We use the term ``MGCN'' to refer to R-GCN, but with the negative weights clamped to zero during training, after every optimizer step. In this way, every channel is trivially safe \changemarker{and rules can be checked for soundness}. This is the same training method used by \cite{cucala2021explainable}, but on a different GNN. We also use early-stopping in our normal training of R-GCNs and MGCNs variation, ceasing training if the model loss deteriorates for $50$ epochs in a row. MGCNs are purely monotonic and are thus unable to learn non-monotonic patterns in the data.

To encourage R-GCNs to exhibit non-monotonic behaviour while simultaneously learning \changemarker{sound} monotonic rules, we propose a novel variation to the training routine. 
Our method relies on matrix weight clamping: \changemarker{intuitively, having more zero matrix weights leads to a higher number of safe, stable, and increasing channels.
Smaller weights affect the actual computation less, so we clamp those before larger ones.}
To clamp a sum-GNN $\mathcal{N}$ using a 
\emph{clamping threshold} $\tau \in \mathbb{R}^+$,
we set to $0$ each model weight $\mathbf{A}_\ell[i, j]$ such that
$|\mathbf{A}_\ell[i, j]| \leq \tau$ (and likewise for each $\mathbf{B}_\ell^c[i, j]$).
Given a parameter $X \in [0, 100]$, let $\tau_X$ be the minimum absolute value of a matrix weight of $\mathcal{N}$ such that the percentage of output channels of $\mathcal{N}$ that are stable or increasing is strictly greater than $X$.
We train three separate instances of each model corresponding to three values of $X$: $0$, $25$, and $50$; for a given value of $X$, 
after each epoch, we compute $\tau_X$ and clamp the model's matrices using $\tau_X$ as a clamping threshold.
We refer to the model trained for $X$ as R-$X$. We apply these models to all datasets except those  derived from FB (LogInfer-FB and FB237v1), as it is prohibitively slow to compute which channels are increasing or stable at every epoch due to the high number of predicates used in the dataset.

\subsection{Results}
\paragraph{Monotonic LogInfer Datasets}
The results on the monotonic LogInfer benchmarks are shown in \Cref{results:monotonic_wn}.
We note that, \changemarker{across every experiment}, R-GCN provably never has any sound rules, since every channel is found to be unbounded.
These results are very surprising, since our monotonic rule patterns are very simple. They are especially
surprising for super-hierarchy, where most positive examples are consequences of simple rules.
This stands in contrast with the near-perfect accuracy that R-GCN obtained on the monotonic LogInfer benchmarks: for example, $99.66\%$ for WN-superhier. This proves that even in cases where accuracy is very high, R-GCN may not have learned any sound rules, which raises doubts about the usefulness of accuracy as a metric for KG completion models, since its high value (in this case) does not \changemarker{correspond to sound rules}.

In \Cref{results:monotonic_wn}, R-GCN consistently achieves a lower loss than MGCN on the training data due to the lack of restrictions. However, it is almost always outperformed by MGCN on the test set. We attribute this to the strong inductive bias that MGCNs have for these particular problems, which require learning monotonic rules.
Furthermore, the rules used to augment the datasets are often sound for the MGCN models, which demonstrates that they have effectively learned the relevant completion patterns for these datasets.

\begin{table*}
\centering
\resizebox{2.11\columnwidth}{!}{
\begin{tabular}{ll|rrrr|rrr|rr|rr}
\toprule
Dataset & Model & \%Acc & \%Prec & \%Rec & Loss & \%UB & \%Stable & \%Inc & \%SO & \%NG & \#1B & \#2B \\
\midrule
WN-hier
& R-GCN & \textbf{98.87} & \textbf{99.32} & 98.42 & \textbf{9172}  & 100 & 0 & 0   & 0   & 0 & 0 & 0 \\
& MGCN  & 98.75 & 97.56 & \textbf{100}   & 40366 & 0   & 0 & 100 & \textbf{100} & 0 & \textbf{31} & \textbf{3095} \\

\midrule

WN-sym
& R-GCN & 98.95 & 99.2 & 98.69 & \textbf{13113} & 100 & 0     & 0     & 0   & 0 & 0 & 0 \\
& MGCN  & \textbf{100}   & \textbf{100}  & \textbf{100}   & 35544 & 0   & 19.09 & 80.91 & \textbf{100} & 0 & \textbf{17} & \textbf{1671} \\

\midrule

WN-superhier
& R-GCN & 99.66 & \textbf{99.7}  & 99.63 & \textbf{21331}  & 100 & 0  & 0  & 0  & 0 & 0 & 0 \\
& MGCN  & \textbf{99.67} & 99.34 & \textbf{100}   & 108789 & 0   & 10 & 90 & \textbf{98} & 0 & \textbf{24} & \textbf{2348} \\

\midrule

FB-hier
& R-GCN & 92.21 & 94.28 & 89.88 & \textbf{17209}  & 100 & 0    & 0     & 0     & 0     & 0 & 0 \\
& MGCN  & \textbf{97.47} & \textbf{98.28} & \textbf{96.65} & 164345 & 0   & 3.08 & 96.92 & \textbf{54.05} & 45.95 & \textbf{523} & - \\

\midrule

FB-sym
& R-GCN & 94.03 & 95.72 & 92.21 & \textbf{26397}  & 100 & 0     & 0     & 0    & 0   & 0 & 0 \\
& MGCN  & \textbf{98.11} & \textbf{96.61} & \textbf{99.72} & 127030 & 0   & 27.22 & 72.78 & \textbf{91.4} & 8.6 & \textbf{1929} & - \\

\midrule

FB-superhier
& R-GCN & 98.87 & 99.25 & 98.48 & \textbf{56247}  & 100 & 0    & 0     & 0    & 0    & 0 & 0 \\
& MGCN  & \textbf{99.67} & \textbf{99.6}  & \textbf{99.75} & 426788 & 0   & 2.15 & 97.85 & \textbf{51.9} & 48.1 & \textbf{307} & - \\

\bottomrule
\end{tabular}
} 
\caption{Results for the monotonic LogInfer datasets. Loss is on the training set. \%UB, \%Stable, and \%Inc are the percentages of unbounded, stable, and increasing channels, respectively. \%SO is the percentage of LogInfer rules that are sound for the GNN and \%NG is the percentage of monotonic LogInfer rules that are not sound for the GNN since there is some grounding of the body which does not entail the head. \#1B and \#2B are the number of sound rules with one and two body atoms respectively.}
\label{results:monotonic_wn}
\end{table*}

\paragraph{Benchmark Datasets}
\begin{table*}
\centering
\resizebox{2.11\columnwidth}{!}{
\begin{tabular}{ll|rrrrr|rrr|r|rr}
\toprule
Dataset & Model & \%Acc & \%Prec & \%Rec & AUPRC & Loss & \%UB & \%Stable & \%Inc & \%Safe & \#1B & \#2B \\
\midrule
FB237v1
& R-GCN  & 68.4  & \textbf{99.86} & 36.85 & 0.1407 & \textbf{24}  & 100 & 0     & 0     & 0   & 0 & 0 \\
& MGCN   & \textbf{74.95} & 82.89 & \textbf{64.05} & \textbf{0.1456} & 765 & 0   & 69.44 & 30.56 & 100 & \textbf{13640} & - \\

\midrule

NELLv1
& R-GCN  & 60.4  & 47.96 & 22.35 & \textbf{0.08}   & \textbf{687}  & 100   & 0     & 0     & 0     & 0 & 0 \\
& R-0    & 61.24 & 83.18 & 24.71 & 0.0769 & 1342 & 92.86 & 2.14  & 5     & 6.43  & 0 & 10 \\
& R-25   & 67.35 & 76.62 & 36.35 & 0.0766 & 2095 & 71.43 & 9.29  & 19.29 & 27.96 & 6 & 799 \\
& R-50   & 71.53 & \textbf{91.99} & 46.59 & 0.0761 & 3148 & 40.71 & 18.57 & 38.57 & 56.43 & 9 & 1055 \\
& MGCN   & \textbf{91.94} & 86.17 & \textbf{100}   & 0.0737 & 2549 & 0     & 14.29 & 85.71 & 100   & \textbf{26} & \textbf{3300} \\

\midrule

WN18RRv1
& R-GCN  & 94    & 96.9  & 90.97 & \textbf{0.1197} & \textbf{941}  & 100   & 0     & 0     & 0     & 0 & 0 \\
& R-0    & 94.52 & \textbf{98.36} & 90.55 & 0.1176 & 1109 & 87.78 & 10    & 1.11  & 11.11 & 0 & 16 \\
& R-25   & 94.09 & 96.82 & 91.45 & 0.1157 & 1544 & 61.11 & 28.89 & 4.44  & 33.33 & 3 & 293 \\
& R-50   & 95.15 & 98.32 & 91.88 & 0.1163 & 1513 & 38.89 & 40    & 15.56 & 55.56 & 3 & 222 \\
& MGCN   & \textbf{95.18} & 97.52 & \textbf{92.73} & 0.114  & 2292 & 0     & 44.44 & 55.56 & 100   & \textbf{7} & \textbf{681} \\

\bottomrule
\end{tabular}
} 
\caption{Results for the benchmark datasets. AUPRC is on the validation set and Loss on the training set. \%UB, \%Stable, \%Inc, and \%Safe are the percentages of unbounded, stable, increasing, and safe channels, respectively. \#1B and \#2B are the number of sound rules with one and two body atoms respectively.}
\label{results:baselines}   
\end{table*}
The findings of our experiments on the benchmark datasets are shown in \Cref{results:baselines}.
AUPRC is consistently higher for R-GCN on the validation set than for MGCN but accuracy on the test set is always higher for MGCN. This is because, for most constant pairs used in the validation set targets, there exists no fact in the 
input dataset where these constants appear together.
Our dataset transformation, however, can only predict a fact 
if the constants in it appear together in a fact of the input dataset. 
This means that there are few data points with which to choose a sensible threshold on the validation set.
Nevertheless, the results illustrate how, as $X$ increases in model R-$X$, the ratios of stable/increasing/safe channels and the number of sound rules increase, whilst the AUPRC generally deteriorates.

\paragraph{Mixed LogInfer Datasets}
\begin{table*}
\centering
\resizebox{2.11\columnwidth}{!}{
\begin{tabular}{ll|rrr|rrr|r|rrr|rr}
\toprule
Dataset & Model & \%Acc & \%Prec & \%Rec & \%UB & \%Stable & \%Inc & \%Safe & \%SO & \%NG & \%NB & \#1B & \#2B \\
\midrule
WN-hier\_nmhier
& R-GCN  & \textbf{89.0} & \textbf{87.65} & \textbf{90.88} & 100   & 0    & 0     & 0     & 0  & 0  & 100 & 0 & 0 \\
& R-0    & 79.48 & 78.18 & 82.33 & 90.91 & 0    & 9.09  & 8.18  & 0  & 0  & 100 & 5 & 513 \\
& R-25   & 71.13 & 75.03 & 66.17 & 72.73 & 0    & 27.27 & 25.45 & 8  & 2  & 90  & 10 & 947 \\
& R-50   & 69.59 & 74.63 & 61.94 & 45.45 & 4.55 & 50    & 51.82 & 30 & 10 & 60  & \textbf{36} & \textbf{3589} \\
& MGCN   & 66.87 & 74.37 & 56.38 & 0     & 0    & 100   & 100   & \textbf{52} & 48 & 0   & 31 & 3042 \\

\midrule

WN-cup\_nmhier
& R-GCN  & \textbf{83.32} & \textbf{83.36} & 83.47 & 100   & 0    & 0     & 0     & 0  & 0  & 100 & 0 & 0 \\
& R-0    & 77.22 & 77.3  & 78.38 & 90.91 & 0.91 & 8.18  & 9.09  & 0  & 0  & 100 & 3 & 320 \\
& R-25   & 71.16 & 69.6  & 76.32 & 72.73 & 5.46 & 21.82 & 26.36 & 10 & 10 & 80  & 20 & 2023 \\
& R-50   & 66.8  & 67.54 & 66.88 & 44.55 & 7.27 & 47.27 & 49.09 & 34 & 16 & 50  & 23 & 2249 \\
& MGCN   & 62.74 & 59.25 & \textbf{85.27} & 0     & 4.55 & 95.45 & 100   & \textbf{50} & 50 & 0   & \textbf{50} & \textbf{4905} \\

\bottomrule
\end{tabular}
} 
\caption{Results for the mixed LogInfer-WN datasets. \%UB, \%Stable, \%Inc, and \%Safe are the percentages of unbounded, stable, increasing, and safe channels, respectively. \%SO is the percentage of monotonic LogInfer rules that are sound for the GNN, \%NG is the percentage of rules that are not sound for the GNN since there is some grounding of the body which does not entail the head, and \%NB the percentage that are not sound since their heads correspond to unbounded channels. \#1B and \#2B are the number of sound rules with one and two body atoms respectively.}
\label{results:mixed_wn}
\end{table*}
The findings on the mixed LogInfer-WN datasets, that is, those that use both monotonic and non-monotonic rule patterns, are shown in \Cref{results:mixed_wn} and demonstrate a clear tendency:
as the number of channels required to be increasing/stable goes up (from R-GCN, to R-$X$, to MGCN), accuracy decreases, whilst the number of sound rules and percentage of sound injected LogInfer rules increase.
This demonstrates a trade-off between interpretability and empirical performance. As in the case of purely monotonic datasets, R-GCN achieves consistently superior accuracy, but provably has no sound rules, since every channel is unbounded.
Notice that on WN-cup\_nmhier, MGCN achieves high recall but low precision. This indicates that the model is learning versions of the non-monotonic rules without the negation, and thus achieving a high recall but getting penalised by false positives, since it is predicting facts that are not derived by the non-monotonic rules.
We observe that the MGCN solution does not scale to every situation, as it can lead to poor performance. Given that many real-world datasets may also have non-monotonic patterns, MGCNs could be too restrictive.

\paragraph{Channel Classification}
As can be seen across all results, the classification of channels into unbounded, stable, and increasing almost always adds up to $100\%$. This demonstrates empirically that, whilst our theoretical analysis allows for the existence of undetermined or decreasing channels that are not unbounded, virtually no such channels are found when using our proposed training paradigms, and so our proposed channel classification allows us to fully determine whether each channel shows monotonic or non-monotonic behaviour. Furthermore, across all experiments on the LogInfer datasets, for every rule used to inject facts, our rule-checking procedures were always able to classify the rule as sound or provably not sound.
Finally, the results show that the total number of channels found to be either stable or increasing is
often larger than the number of channels shown to be safe. 
This can be seen by computing the sum of the \%Stable and \%Inc columns and comparing it to \%Safe.
Although the difference is not big, this validates the need for the more complex classification of channels into stable/increasing/decreasing/undetermined, instead of the simpler classification into safe/unsafe.
\section{Conclusion} \label{sec:conclusions}
In this paper, we provided two ways to extract sound rules from sum-GNNs and a procedure to prove that certain rules are not sound for a sum-GNN. Our methods rely on classifying the output channels of the sum-GNN. We found that, in our experiments, R-GCN, a specific instance of sum-GNN, provably has no sound rules when trained in \changemarker{the standard} way, even when using ideal datasets.
We provided two alternatives to train R-GCN, the first of which clamps \changemarker{all} negative weights to zero and results in R-GCN being entirely monotonic, yielding good performance and many sensible sound rules on datasets with monotonic patterns, but poor performance on datasets with a mixture of monotonic and non-monotonic patterns.
Our second alternative yields a trade-off between accuracy and the number of sound rules. We found that, in practice, almost every channel of the sum-GNN is classified in a manner that allows us to either extract sound rules or prove that there are no associated sound rules.

\changemarker{
The limitations of this work are as follows. First, our analysis only considers GNNs with sum aggregation and monotonically increasing activation functions: it is unclear how sound rules can be extracted from models that use mean aggregation or GELU activation functions. Furthermore, our methods do not guarantee that every output channel will be characterised such that we can show that it either has no sound rules or it can be checked for sound rules.
}

For future work, we aim to consider other GNN architectures, extend our rule extraction to non-monotonic logics, \changemarker{provide relaxed definitions of soundness,} and consider R-GCN with a scoring function (such as DistMult) as a decoder, instead of using a dataset transformation to perform link prediction.

\clearpage
\bibliographystyle{kr}
\bibliography{refs}

\begin{thebibliography}{}

\bibitem[\protect\citeauthoryear{Abboud \bgroup et al\mbox.\egroup }{2020}]{DBLP:conf/nips/AbboudCLS20}
Abboud, R.; Ceylan, {\.I}.~{\.I}.; Lukasiewicz, T.; and Salvatori, T.
\newblock 2020.
\newblock Boxe: {A} box embedding model for knowledge base completion.
\newblock In Larochelle, H.; Ranzato, M.; Hadsell, R.; Balcan, M.; and Lin, H., eds., {\em Advances in Neural Information Processing Systems 33: Annual Conference on Neural Information Processing Systems 2020, NeurIPS 2020, December 6-12, 2020, virtual}.

\bibitem[\protect\citeauthoryear{Bordes \bgroup et al\mbox.\egroup }{2013}]{bordes2013translating}
Bordes, A.; Usunier, N.; Garcia-Duran, A.; Weston, J.; and Yakhnenko, O.
\newblock 2013.
\newblock Translating embeddings for modeling multi-relational data.
\newblock {\em Advances in neural information processing systems} 26.

\bibitem[\protect\citeauthoryear{Bouchard, Singh, and Trouillon}{2015}]{bouchard2015approximate}
Bouchard, G.; Singh, S.; and Trouillon, T.
\newblock 2015.
\newblock On approximate reasoning capabilities of low-rank vector spaces.
\newblock In {\em 2015 AAAI Spring Symposium Series}.

\bibitem[\protect\citeauthoryear{Cai \bgroup et al\mbox.\egroup }{2019}]{cai2019transgcn}
Cai, L.; Yan, B.; Mai, G.; Janowicz, K.; and Zhu, R.
\newblock 2019.
\newblock Transgcn: Coupling transformation assumptions with graph convolutional networks for link prediction.
\newblock In {\em Proceedings of the 10th international conference on knowledge capture},  131--138.

\bibitem[\protect\citeauthoryear{Dettmers \bgroup et al\mbox.\egroup }{2018}]{dettmers2018convolutional}
Dettmers, T.; Minervini, P.; Stenetorp, P.; and Riedel, S.
\newblock 2018.
\newblock Convolutional 2d knowledge graph embeddings.
\newblock In {\em Proceedings of the AAAI conference on artificial intelligence}, volume~32.

\bibitem[\protect\citeauthoryear{Evans and Grefenstette}{2018}]{evans2018learning}
Evans, R., and Grefenstette, E.
\newblock 2018.
\newblock Learning explanatory rules from noisy data.
\newblock {\em Journal of Artificial Intelligence Research} 61:1--64.

\bibitem[\protect\citeauthoryear{Garnelo and Shanahan}{2019}]{garnelo2019reconciling}
Garnelo, M., and Shanahan, M.
\newblock 2019.
\newblock Reconciling deep learning with symbolic artificial intelligence: representing objects and relations.
\newblock {\em Current Opinion in Behavioral Sciences} 29:17--23.

\bibitem[\protect\citeauthoryear{Gutteridge \bgroup et al\mbox.\egroup }{2023}]{gutteridge2023drew}
Gutteridge, B.; Dong, X.; Bronstein, M.~M.; and Di~Giovanni, F.
\newblock 2023.
\newblock Drew: Dynamically rewired message passing with delay.
\newblock In {\em International Conference on Machine Learning},  12252--12267.
\newblock PMLR.

\bibitem[\protect\citeauthoryear{Hamilton, Ying, and Leskovec}{2017}]{hamilton2017inductive}
Hamilton, W.; Ying, Z.; and Leskovec, J.
\newblock 2017.
\newblock Inductive representation learning on large graphs.
\newblock {\em Advances in neural information processing systems} 30.

\bibitem[\protect\citeauthoryear{Hogan \bgroup et al\mbox.\egroup }{2022}]{DBLP:journals/csur/HoganBCdMGKGNNN21}
Hogan, A.; Blomqvist, E.; Cochez, M.; d'Amato, C.; de~Melo, G.; Gutierrez, C.; Kirrane, S.; Gayo, J. E.~L.; Navigli, R.; Neumaier, S.; Ngomo, A.~N.; Polleres, A.; Rashid, S.~M.; Rula, A.; Schmelzeisen, L.; Sequeda, J.~F.; Staab, S.; and Zimmermann, A.
\newblock 2022.
\newblock Knowledge graphs.
\newblock {\em {ACM} Comput. Surv.} 54(4):71:1--71:37.

\bibitem[\protect\citeauthoryear{Ioannidis, Marques, and Giannakis}{2019}]{ioannidis2019recurrent}
Ioannidis, V.~N.; Marques, A.~G.; and Giannakis, G.~B.
\newblock 2019.
\newblock A recurrent graph neural network for multi-relational data.
\newblock In {\em ICASSP 2019-2019 IEEE International Conference on Acoustics, Speech and Signal Processing (ICASSP)},  8157--8161.
\newblock IEEE.

\bibitem[\protect\citeauthoryear{Li \bgroup et al\mbox.\egroup }{2023}]{li2023skier}
Li, W.; Zhu, L.; Mao, R.; and Cambria, E.
\newblock 2023.
\newblock Skier: A symbolic knowledge integrated model for conversational emotion recognition.
\newblock In {\em Proceedings of the AAAI Conference on Artificial Intelligence}, volume~37,  13121--13129.

\bibitem[\protect\citeauthoryear{Lin \bgroup et al\mbox.\egroup }{2023}]{lin2023fusing}
Lin, Q.; Mao, R.; Liu, J.; Xu, F.; and Cambria, E.
\newblock 2023.
\newblock Fusing topology contexts and logical rules in language models for knowledge graph completion.
\newblock {\em Information Fusion} 90:253--264.

\bibitem[\protect\citeauthoryear{Liu \bgroup et al\mbox.\egroup }{2021}]{liu2021indigo}
Liu, S.; Grau, B.; Horrocks, I.; and Kostylev, E.
\newblock 2021.
\newblock Indigo: Gnn-based inductive knowledge graph completion using pair-wise encoding.
\newblock {\em Advances in Neural Information Processing Systems} 34:2034--2045.

\bibitem[\protect\citeauthoryear{Liu \bgroup et al\mbox.\egroup }{2023}]{liu2023revisiting}
Liu, S.; Grau, B.~C.; Horrocks, I.; and Kostylev, E.~V.
\newblock 2023.
\newblock Revisiting inferential benchmarks for knowledge graph completion.
\newblock In {\em Proceedings of the International Conference on Principles of Knowledge Representation and Reasoning}, volume~19,  461--471.

\bibitem[\protect\citeauthoryear{Meilicke \bgroup et al\mbox.\egroup }{2018}]{meilicke2018fine}
Meilicke, C.; Fink, M.; Wang, Y.; Ruffinelli, D.; Gemulla, R.; and Stuckenschmidt, H.
\newblock 2018.
\newblock Fine-grained evaluation of rule-and embedding-based systems for knowledge graph completion.
\newblock In {\em The Semantic Web--ISWC 2018: 17th International Semantic Web Conference, Monterey, CA, USA, October 8--12, 2018, Proceedings, Part I 17},  3--20.
\newblock Springer.

\bibitem[\protect\citeauthoryear{Nickel \bgroup et al\mbox.\egroup }{2011}]{nickel2011three}
Nickel, M.; Tresp, V.; Kriegel, H.-P.; et~al.
\newblock 2011.
\newblock A three-way model for collective learning on multi-relational data.
\newblock In {\em Icml}, volume~11,  3104482--3104584.

\bibitem[\protect\citeauthoryear{Pflueger, {Tena Cucala}, and Kostylev}{2022}]{pflueger2022gnnq}
Pflueger, M.; {Tena Cucala}, D.~J.; and Kostylev, E.~V.
\newblock 2022.
\newblock Gnnq: A neuro-symbolic approach to query answering over incomplete knowledge graphs.
\newblock In {\em International Semantic Web Conference},  481--497.
\newblock Springer.

\bibitem[\protect\citeauthoryear{Qu \bgroup et al\mbox.\egroup }{2020}]{qu2020rnnlogic}
Qu, M.; Chen, J.; Xhonneux, L.-P.; Bengio, Y.; and Tang, J.
\newblock 2020.
\newblock Rnnlogic: Learning logic rules for reasoning on knowledge graphs.
\newblock In {\em International Conference on Learning Representations}.

\bibitem[\protect\citeauthoryear{Rockt{\"a}schel and Riedel}{2017}]{rocktaschel2017end}
Rockt{\"a}schel, T., and Riedel, S.
\newblock 2017.
\newblock End-to-end differentiable proving.
\newblock {\em Advances in neural information processing systems} 30.

\bibitem[\protect\citeauthoryear{Sadeghian \bgroup et al\mbox.\egroup }{2019}]{sadeghian2019drum}
Sadeghian, A.; Armandpour, M.; Ding, P.; and Wang, D.~Z.
\newblock 2019.
\newblock Drum: End-to-end differentiable rule mining on knowledge graphs.
\newblock {\em Advances in Neural Information Processing Systems} 32.

\bibitem[\protect\citeauthoryear{Schlichtkrull \bgroup et al\mbox.\egroup }{2018}]{schlichtkrull2018modeling}
Schlichtkrull, M.; Kipf, T.~N.; Bloem, P.; Van Den~Berg, R.; Titov, I.; and Welling, M.
\newblock 2018.
\newblock Modeling relational data with graph convolutional networks.
\newblock In {\em The Semantic Web: 15th International Conference, ESWC 2018, Heraklion, Crete, Greece, June 3--7, 2018, Proceedings 15},  593--607.
\newblock Springer.

\bibitem[\protect\citeauthoryear{Shang \bgroup et al\mbox.\egroup }{2019}]{shang2019end}
Shang, C.; Tang, Y.; Huang, J.; Bi, J.; He, X.; and Zhou, B.
\newblock 2019.
\newblock End-to-end structure-aware convolutional networks for knowledge base completion.
\newblock In {\em Proceedings of the AAAI conference on artificial intelligence}, volume~33,  3060--3067.

\bibitem[\protect\citeauthoryear{Suchanek, Kasneci, and Weikum}{2007}]{suchanek2007yago}
Suchanek, F.~M.; Kasneci, G.; and Weikum, G.
\newblock 2007.
\newblock Yago: a core of semantic knowledge.
\newblock In {\em Proceedings of the 16th international conference on World Wide Web},  697--706.

\bibitem[\protect\citeauthoryear{Sun \bgroup et al\mbox.\egroup }{2018}]{sun2018rotate}
Sun, Z.; Deng, Z.-H.; Nie, J.-Y.; and Tang, J.
\newblock 2018.
\newblock Rotate: Knowledge graph embedding by relational rotation in complex space.
\newblock In {\em International Conference on Learning Representations}.

\bibitem[\protect\citeauthoryear{Tang \bgroup et al\mbox.\egroup }{2024}]{tang2024gadbench}
Tang, J.; Hua, F.; Gao, Z.; Zhao, P.; and Li, J.
\newblock 2024.
\newblock Gadbench: Revisiting and benchmarking supervised graph anomaly detection.
\newblock {\em Advances in Neural Information Processing Systems} 36.

\bibitem[\protect\citeauthoryear{{Tena Cucala} \bgroup et al\mbox.\egroup }{2021}]{cucala2021explainable}
{Tena Cucala}, D.; {Cuenca Grau}, B.; Kostylev, E.~V.; and Motik, B.
\newblock 2021.
\newblock Explainable gnn-based models over knowledge graphs.
\newblock In {\em International Conference on Learning Representations}.

\bibitem[\protect\citeauthoryear{{Tena Cucala} \bgroup et al\mbox.\egroup }{2023}]{cucala2023correspondence}
{Tena Cucala}, D.; {Cuenca Grau}, B.; Motik, B.; and Kostylev, E.~V.
\newblock 2023.
\newblock {On the Correspondence Between Monotonic Max-Sum GNNs and Datalog}.
\newblock In {\em {Proceedings of the 20th International Conference on Principles of Knowledge Representation and Reasoning}},  658--667.

\bibitem[\protect\citeauthoryear{{Tena Cucala}, {Cuenca Grau}, and Motik}{2022}]{cucala2022faithful}
{Tena Cucala}, D.~J.; {Cuenca Grau}, B.; and Motik, B.
\newblock 2022.
\newblock Faithful approaches to rule learning.
\newblock In {\em Proceedings of the International Conference on Principles of Knowledge Representation and Reasoning}, volume~19,  484--493.

\bibitem[\protect\citeauthoryear{Teru, Denis, and Hamilton}{2020}]{teru2020inductive}
Teru, K.; Denis, E.; and Hamilton, W.
\newblock 2020.
\newblock Inductive relation prediction by subgraph reasoning.
\newblock In {\em International Conference on Machine Learning},  9448--9457.
\newblock PMLR.

\bibitem[\protect\citeauthoryear{Tian \bgroup et al\mbox.\egroup }{2020}]{tian2020ra}
Tian, A.; Zhang, C.; Rang, M.; Yang, X.; and Zhan, Z.
\newblock 2020.
\newblock Ra-gcn: Relational aggregation graph convolutional network for knowledge graph completion.
\newblock In {\em Proceedings of the 2020 12th international conference on machine learning and computing},  580--586.

\bibitem[\protect\citeauthoryear{Toutanova and Chen}{2015}]{toutanova2015observed}
Toutanova, K., and Chen, D.
\newblock 2015.
\newblock Observed versus latent features for knowledge base and text inference.
\newblock In {\em Proceedings of the 3rd workshop on continuous vector space models and their compositionality},  57--66.

\bibitem[\protect\citeauthoryear{Vashishth \bgroup et al\mbox.\egroup }{2019}]{vashishth2019composition}
Vashishth, S.; Sanyal, S.; Nitin, V.; and Talukdar, P.
\newblock 2019.
\newblock Composition-based multi-relational graph convolutional networks.
\newblock In {\em International Conference on Learning Representations}.

\bibitem[\protect\citeauthoryear{Vrande{\v{c}}i{\'c} and Kr{\"o}tzsch}{2014}]{vrandevcic2014wikidata}
Vrande{\v{c}}i{\'c}, D., and Kr{\"o}tzsch, M.
\newblock 2014.
\newblock Wikidata: a free collaborative knowledgebase.
\newblock {\em Communications of the ACM} 57(10):78--85.

\bibitem[\protect\citeauthoryear{Wang \bgroup et al\mbox.\egroup }{2023}]{wang2023faithful}
Wang, X.; {Tena Cucala}, D.; {Cuenca Grau}, B.; and Horrocks, I.
\newblock 2023.
\newblock Faithful rule extraction for differentiable rule learning models.
\newblock In {\em The Twelfth International Conference on Learning Representations}.

\bibitem[\protect\citeauthoryear{Xie \bgroup et al\mbox.\egroup }{2022}]{xie2022discrimination}
Xie, X.; Zhang, N.; Li, Z.; Deng, S.; Chen, H.; Xiong, F.; Chen, M.; and Chen, H.
\newblock 2022.
\newblock From discrimination to generation: Knowledge graph completion with generative transformer.
\newblock In {\em Companion Proceedings of the Web Conference 2022},  162--165.

\bibitem[\protect\citeauthoryear{Yang \bgroup et al\mbox.\egroup }{2015}]{yang2015embedding}
Yang, B.; Yih, S. W.-t.; He, X.; Gao, J.; and Deng, L.
\newblock 2015.
\newblock Embedding entities and relations for learning and inference in knowledge bases.
\newblock In {\em Proceedings of the International Conference on Learning Representations (ICLR) 2015}.

\bibitem[\protect\citeauthoryear{Yang, Yang, and Cohen}{2017}]{yang2017differentiable}
Yang, F.; Yang, Z.; and Cohen, W.~W.
\newblock 2017.
\newblock Differentiable learning of logical rules for knowledge base reasoning.
\newblock {\em Advances in neural information processing systems} 30.

\bibitem[\protect\citeauthoryear{Yao, Mao, and Luo}{2019}]{yao2019kg}
Yao, L.; Mao, C.; and Luo, Y.
\newblock 2019.
\newblock Kg-bert: Bert for knowledge graph completion.
\newblock {\em arXiv preprint arXiv:1909.03193}.

\bibitem[\protect\citeauthoryear{Yu \bgroup et al\mbox.\egroup }{2021}]{yu2021knowledge}
Yu, D.; Yang, Y.; Zhang, R.; and Wu, Y.
\newblock 2021.
\newblock Knowledge embedding based graph convolutional network.
\newblock In {\em Proceedings of the Web Conference 2021},  1619--1628.

\bibitem[\protect\citeauthoryear{Zhang \bgroup et al\mbox.\egroup }{2023}]{zhang2023learning}
Zhang, M.; Xia, Y.; Liu, Q.; Wu, S.; and Wang, L.
\newblock 2023.
\newblock Learning latent relations for temporal knowledge graph reasoning.
\newblock In {\em Proceedings of the 61st Annual Meeting of the Association for Computational Linguistics (Volume 1: Long Papers)},  12617--12631.

\end{thebibliography}

For the purpose of Open Access, the authors have applied a CC BY public copyright licence to any Author Accepted Manuscript (AAM) version arising from this submission.

\paragraph{Acknowledgments}
Matthew Morris is funded by an EPSRC scholarship (CS2122\_EPSRC\_1339049).
This work was also supported by the SIRIUS Centre for Scalable Data Access (Research Council of Norway, project 237889), Samsung Research UK, the EPSRC projects AnaLOG (EP/P025943/1), OASIS (EP/S032347/1), UKFIRES (EP/S019111/1) and ConCur (EP/V050869/1).
The authors would like to acknowledge the use of the University of Oxford Advanced Research Computing (ARC) facility in carrying out this work \url{http://dx.doi.org/10.5281/zenodo.22558}.

\appendix
\clearpage

\section{Full Proofs}

\subsection{Proposition \ref{prop:sound_finite_containment}} \label{app:sound_finite_containment_proof}

\begin{proposition*}
\changemarker{
If $\alpha$ is a rule or program sound for sum-GNN $\mathcal{N}$, then for any dataset $D$ and $k \in \mathbb{N}$, the containment holds when $T_\alpha$ and $T_\mathcal{N}$ are composed $k$ times:
$T_\alpha^k(D) \subseteq T_{\mathcal{N}}^k(D)$.
}
\end{proposition*}

\begin{proof}
Let $D$ be an arbitrary dataset.
We prove the claim by induction on $k$.
The base case of $k = 1$ follows trivially, since $T_\alpha(D) \subseteq T_{\mathcal{N}}(D)$ holds from the soundness of $\alpha$.

For the inductive step, let $k \geq 2$.

Let $D_\alpha = T_\alpha^{k-1}(D)$ and $D_\mathcal{N} = T_{\mathcal{N}}^{k-1}(D)$.
Then from the soundness of $\alpha$,
$T_\alpha(D_\mathcal{N}) \subseteq T_\mathcal{N}(D_\mathcal{N})$,
and from the induction hypothesis,
$D_\alpha \subseteq D_\mathcal{N}$.

But then $T_\alpha(D_\alpha) \subseteq T_\alpha(D_\mathcal{N})$, since operator $T_\alpha$ is monotonic.
So we have $T_\alpha(D_\alpha) \subseteq T_\alpha(D_\mathcal{N}) \subseteq T_\mathcal{N}(D_\mathcal{N})$, and thus
$T_\alpha^k(D) \subseteq T_\mathcal{N}^k(D)$.
\end{proof}

\subsection{Lemma \ref{lemma:updown_monotonic}} \label{app:updown_monotonic_proof}
\begin{lemma*}
Let $\mathcal{N}$ be a sum-GNN of $L$ layers, and let $D', D$ be datasets satisfying $D' \subseteq D$.
For each vertex $v \in \texttt{enc}(D')$, layer $\ell \in \{0, \dots, L\}$, and channel $i \in \{1, \dots, \delta_{\ell}\}$,
the following hold:
\begin{itemize}
\item If $i$ is stable at layer $\ell$, then $\mathbf{v}_\ell'[i] = \mathbf{v}_\ell[i]$;
\item If $i$ is increasing at layer $\ell$, then $\mathbf{v}_\ell'[i] \leq \mathbf{v}_\ell[i]$, and 
\item If $i$ is decreasing at layer $\ell$, then $\mathbf{v}_\ell'[i] \geq \mathbf{v}_\ell[i]$,
\end{itemize}
where $\mathbf{v}_\ell$ and $\mathbf{v}_\ell'$ are the vectors induced for $v$ in layer $\ell$ by applying $\mathcal{N}$ to $D$ and $D'$ respectively.
\end{lemma*}

\begin{proof}
We show the claim of the lemma by induction on $\ell$. The base case $\ell=0$ holds trivially by the definition of the canonical encoding, the fact that $D' \subseteq D$, and all channels being increasing at layer $\ell=0$.

For the inductive step, we assume that the claim holds for $\ell -1\geq 0$. We prove that it holds at layer $\ell$ in three parts. Defining $E^c, (E^c)'$ to be the edges of $\texttt{enc}(D), \texttt{enc}(D')$ respectively, we have that $(E^c)' \subseteq E^c$.

\begin{center}
\textbf{Stable Channels}
\end{center}

Assume that channel $i$ is stable at layer $\ell$. We show that $\mathbf{v}_\ell'[i] = \mathbf{v}_\ell[i]$. Consider the computation of $\mathbf{v}_\ell[i]$ and $\mathbf{v}_\ell'[i]$ as per \Cref{align:def_gnn_update}. The value $(\mathbf{A}_\ell \mathbf{v}_{\ell-1})[i]$ is the sum over all $j \in \{ 1, ..., \delta_{\ell-1} \}$ of $\mathbf{A}_\ell[i, j] \mathbf{v}_{\ell-1}[j]$. But since $i$ is stable, each such $j$ that is unstable at layer $\ell-1$ has $\mathbf{A}_\ell[i, j] = 0$.

Thus, $(\mathbf{A}_\ell \mathbf{v}_{\ell-1})[i]$ is equal to the sum over each channel $j$ that is stable at layer $\ell-1$.
By induction, each such stable $j$ satisfies $\mathbf{v}_{\ell-1}'[j] = \mathbf{v}_{\ell-1}[j]$. Thus $\mathbf{A}_\ell[i, j] \mathbf{v}_{\ell-1}'[j] = \mathbf{A}_\ell[i, j] \mathbf{v}_{\ell-1}[j]$, and therefore $(\mathbf{A}_\ell \mathbf{v}_{\ell-1}')[i] = (\mathbf{A}_\ell \mathbf{v}_{\ell-1})[i]$.

Furthermore, for all $j$ (both stable and unstable), $\mathbf{B}_\ell^c[i, j] = 0$ for each colour $c \in \text{$\col$}$. Thus, we can derive
$ (\sum_{c \in \text{$\col$}} \mathbf{B}_\ell^c \sum_{(v,u) \in (E^c)'} \mathbf{u}'_{\ell-1} )[i] = 0$
and
$(\sum_{c \in \text{$\col$}} \mathbf{B}_\ell^c \sum_{(v,u) \in E^c} \mathbf{u}_{\ell-1} )[i] = 0$.

This means that $\mathbf{v}_\ell'[i] = \mathbf{v}_\ell[i]$, as required.

\begin{center}
\textbf{Increasing Channels}
\end{center}

Assume that channel $i$ is increasing at layer $\ell$. We show that $\mathbf{v}_\ell'[i] \leq \mathbf{v}_\ell[i]$. As before, consider the computation of $\mathbf{v}_\ell[i]$ and $\mathbf{v}_\ell'[i]$ as per \Cref{align:def_gnn_update}. The value $(\mathbf{A}_\ell \mathbf{v}_{\ell-1})[i]$ is the sum over all $j \in \{ 1, ..., \delta_{\ell-1} \}$ of $\mathbf{A}_\ell[i, j] \mathbf{v}_{\ell-1}[j]$.

Now consider the four possibilities for $j$ at $\ell-1$ and conditions (1) - (3) in the definitions. If $j$ is stable, then by induction, $v_{\ell-1}'[j] = v_{\ell-1}[j]$. If $j$ is increasing, then from (1) we have $\mathbf{A}_\ell[i,j] \geq 0$ and by induction $v_{\ell-1}'[j] \leq v_{\ell-1}[j]$. If $j$ is decreasing, then from (2) we have $\mathbf{A}_\ell[i,j] \leq 0$ and again by induction $v_{\ell-1}'[j] \geq v_{\ell-1}[j]$. Finally, if $j$ is undetermined, then from (3) we obtain $v_{\ell-1}'[j] = 0 = v_{\ell-1}[j]$.

In any of these cases, we find that $\mathbf{A}_\ell[i, j] \mathbf{v}_{\ell-1}'[j] \leq \mathbf{A}_\ell[i, j] \mathbf{v}_{\ell-1}[j]$, and thus we can conclude that $(\mathbf{A}_\ell \mathbf{v}_{\ell-1}')[i] \leq (\mathbf{A}_\ell \mathbf{v}_{\ell-1})[i]$.

Now let $c \in \text{$\col$}$ and consider the matrix product involving $\mathbf{B}_\ell^c$, which is a sum over all $j \in \{ 1, ..., \delta_{\ell-1} \}$. If $j$ is decreasing or undetermined at $\ell-1$ then from (2) and (3) we have that $\mathbf{B}_\ell^c[i, j] = 0$. Then trivially

$$ \mathbf{B}_\ell^c[i, j] ( \sum_{(v,u) \in (E^c)'} \mathbf{u}'_{\ell-1} )[j] \leq \mathbf{B}_\ell^c[i, j] ( \sum_{(v,u) \in E^c} \mathbf{u}_{\ell-1} )[j] . $$

On the other hand, if $\mathbf{B}_\ell^c[i, j] \not= 0$, then by elimination $j$ must be increasing or stable at $\ell-1$, which by induction means that $\mathbf{u}'_{\ell-1}[j] \leq \mathbf{u}_{\ell-1}[j]$ for every node $u \in \texttt{enc}(D')$.

Then since $(E^c)' \subseteq E^c$, we can conclude that $ (\sum_{(v,u) \in (E^c)'} \mathbf{u}'_{\ell-1} )[j] \leq ( \sum_{(v,u) \in E^c} \mathbf{u}_{\ell-1} )[j]$. Note that both of these sums are positive. Also, since from (4) we have that $\mathbf{B}_\ell^c[i, j] \geq 0$, we find that for all $j$,

$$ \mathbf{B}_\ell^c[i, j] ( \sum_{(v,u) \in (E^c)'} \mathbf{u}'_{\ell-1} )[j] \leq \mathbf{B}_\ell^c[i, j] ( \sum_{(v,u) \in E^c} \mathbf{u}_{\ell-1} )[j] . $$

$$ ( \mathbf{B}_\ell^c \sum_{(v,u) \in (E^c)'} \mathbf{u}'_{\ell-1} )[i] \leq ( \mathbf{B}_\ell^c \sum_{(v,u) \in E^c} \mathbf{u}_{\ell-1} )[i] . $$

and thus that

$$ (\sum_{c \in \text{$\col$}} \mathbf{B}_\ell^c \sum_{(v,u) \in (E^c)'} \mathbf{u}'_{\ell-1} )[i] \leq (\sum_{c \in \text{$\col$}} \mathbf{B}_\ell^c \sum_{(v,u) \in E^c} \mathbf{u}_{\ell-1} )[i] . $$

Bringing the above together, along with the fact that $\sigma_\ell$ is monotonically increasing, we obtain $\mathbf{v}_\ell'[i] \leq \mathbf{v}_\ell[i]$.

\begin{center}
\textbf{Decreasing Channels}
\end{center}

This argument is very similar to the one for increasing channels, as the definitions are almost symmetric. If channel $i$ is decreasing at layer $\ell$, we obtain that $\mathbf{v}_\ell'[i] \leq \mathbf{v}_\ell[i]$.

This concludes the proof by induction.
\end{proof}

\subsection{Proposition \ref{prop:updown_sound}} \label{app:updown_sound_proof}
\begin{proposition*}
Let $\mathcal{N}$ be a sum-GNN as in \Cref{eq:GNN}, and let $r$  be a rule of the form \eqref{eq:ruleform} where $H$ mentions a unary predicate $U_p$. Let $S$ be an arbitrary set of as many constants as there are variables in $r$.
Assume channel $p$ in $\mathcal{N}$ is increasing or stable at layer $L$.
Then $r$ is sound for $T_\mathcal{N}$ if and only if
$H\nu \in T_\mathcal{N}(D_r^{\nu})$
for each substitution $\nu$ mapping the variables of $r$ to constants in $S$ and such that $D_r^{\nu} \models B_i\nu$ for each inequality $B_i$ in r.
\end{proposition*}

\begin{proof}

To prove the soundness of $r$, consider an arbitrary dataset $D$. We show that $T_r(D) \subseteq T_\mathcal{N}(D)$. To this end, we consider an arbitrary fact in $T_r(D)$ and show that it is also contained in 
$T_\mathcal{N}(D)$. By the definition of the immediate consequence operator $T_r$, this fact is of the form $H\mu$, where $\mu$ is a substitution from the variables of $r$ to constants 
in $D$ satisfying $D \models B_i\mu$ for each body literal $B_i$ of $r$. Let 
$\sigma$ be an arbitrary one-to-one mapping from the co-domain of $\mu$ to some subset of $S$; such a mapping exists because $S$ has as least as many constants as variables in $r$. Let $\nu$ be the composition of $\mu$ and $\sigma$.

Observe that for each body inequality $B_i$ of $r$, we have $D_r^{\nu} \models B_i \nu$
because $D \models B_i \mu$
and $\sigma$ is injective. 
Therefore, by the assumption of the proposition, $H\nu \in T_{\mathcal{N}}(D_r^{\nu})$.
Now, observe that the result of applying $T_{\mathcal{N}}$ to a dataset does not depend on the identity of the constants, but only on the structure of the dataset; therefore, $T_{\mathcal{N}}$ is invariant under one-to-one mappings of constants, and so $H\nu \in T_{\mathcal{N}}(D_r^{\nu})$ implies
$H\mu \in T_{\mathcal{N}}(D_r^{\mu})$.
Now, let $a$ be the single constant in $H \mu$.
Since $D^r_{\mu} \subseteq D$ by definition of $\mu$, and channel $p$ is increasing or stable at layer $L$, 
we can apply \Cref{lemma:updown_monotonic} to conclude that $\mathbf{v}_L'[p] \leq \mathbf{v}_L[p]$, for $v$ the node corresponding to $a$ in $\texttt{enc}(D^r_{\mu})$. But $H \mu \in T_\mathcal{N}(B \mu)$ implies that $\texttt{cls}_t(\mathbf{v}_L'[p])=1$ and so 
$\mathbf{v}_L'[p] \geq t$.
Hence, $\mathbf{v}_L[p] \geq t$, and so
$\texttt{cls}_t(\mathbf{v}_L[p])=1$,
which implies that $H \mu \in T_\mathcal{N}(D)$, as
we wanted to show.

If on the other hand, $H\nu \not \in T_\mathcal{N}(D_r^\nu)$ for some $\nu$, then $T_r(D_r^\nu) \not\subseteq T_\mathcal{N}(D_r^\nu)$, since $H\nu \in T_r(D_r^\nu)$. Thus, $r$ is not sound for $T_\mathcal{N}$.
\end{proof}

\subsection{Theorem \ref{thm:channel_relations}} \label{app:channel_relations_proof}
\begin{theorem*}
\changemarker{
Safe channels are increasing or stable. Increasing channels are not decreasing. There exist increasing unsafe channels and stable unsafe channels.
}
\end{theorem*}

\begin{proof}
Consider a sum-GNN $\mathcal{N}$ as in \Cref{eq:GNN}. We prove the Theorem in three parts.

\begin{center}
\textbf{Safe channels are increasing or stable}
\end{center}

We prove this by induction on $\ell \in \{ 0, ..., L \}$: we show that if channel $i \in \{ 1, ..., \delta_\ell \}$ is safe at layer $\ell \in \{ 0, ..., L \}$, then it is increasing or stable at layer $\ell$.

Base case: $\ell = 0$. This is trivial, since all channels are increasing at layer $0$.

For the inductive step, consider some layer $\ell \in \{ 1, ..., L \}$ such that channel $i \in \{ 1, ..., \delta_\ell \}$ is safe at layer $\ell$.
Then the $i$-th row of each matrix $\mathbf{A}_\ell$ and $\{(\mathbf{B}_\ell^c)\}_{c \in \col}$ contains only non-negative values and, additionally, the $j$-th element in each such row is zero for each $j \in \{1, \dots, \delta_{\ell-1}\}$ such that $j$ is unsafe in layer $\ell-1$.

We prove that if $i$ is not stable, then it is increasing. To this end, we need to show that
conditions (1) - (4) in \Cref{def:updown} are satisfied for each $j \in \{ 1, ..., \delta_{\ell - 1} \}$. Condition (4) holds for any $j$, since the $i$-th row of each matrix $\{(\mathbf{B}_\ell^c)\}_{c \in \col}$ contains only non-negative values.

If $j$ is safe at layer $\ell-1$ then by the inductive hypothesis it is increasing or stable at layer $\ell-1$. So conditions (2) and (3) hold vacuously. Condition (1) holds since its consequent holds.

If on the other hand, $j$ is unsafe at layer $\ell-1$, then we have $\mathbf{A}_\ell[i, j] = 0$ and $\mathbf{B}_\ell^c[i, j] = 0$ for each $c \in \col$. So conditions (1) - (3) are satisfied by virtue of their consequents holding.

\begin{center}
\textbf{Increasing channels are not decreasing}
\end{center}

Consider a channel $i \in \{ 1, ..., \delta_\ell \}$ at layer $\ell \in \{ 1, ..., L \}$. Assume to the contrary that $i$ is increasing and decreasing at layer $\ell$.

Then from \Cref{def:updown}, $i$ is not stable.
Furthermore, for each $j \in \{ 1, ..., \delta_{\ell - 1} \}$, if $j$ is not stable at layer $\ell - 1$, then it is increasing, decreasing, or undetermined. 
Thus, from conditions (1) - (3) in the definition, we have that for all non-stable $j \in \{ 1, ..., \delta_{\ell - 1} \}$, $\mathbf{A}_\ell[i, j] = 0$, since $i$ is both increasing and decreasing.

Finally, from condition (4), we have that $\mathbf{B}_\ell^c[i, j] = 0$ for each $c \in \col$ and each $j \in \{ 1, ..., \delta_{\ell - 1} \}$.
This satisfies the conditions for stability in \Cref{def:updown}, implying that channel $i$ is stable at layer $\ell$, which is a contradiction.

\begin{center}
\textbf{There exist increasing unsafe channels and stable unsafe channels}
\end{center}

For simplicity, we provide examples of sum-GNNs with such unsafe channels using 2-layer GNNs with $\delta_0 = \delta_1 = \delta_2 = 2$ and $|\col| = 1$. However, the examples can be extended to sum-GNNs with an arbitrary number of layers, arbitrary $\delta$, and arbitrary number of colours.

Let $c \in \col$ be the unique colour in $\col$.
First, to show that there exist unsafe channels that are increasing, consider a 2-layer sum-GNN with the the following matrices. Let $\mathbf{B}_{1}^c = \mathbf{B}_{2}^c$ be $2 \times 2$ matrices filled with zeroes. Define $\mathbf{A}_{1}$ and $\mathbf{A}_{2}$ as follows:

$$
\mathbf{A}_{1} = \begin{pmatrix}
-1 & -1 \\
-1 & -1 \\
\end{pmatrix},~~~~~~~
\mathbf{A}_{2} = \begin{pmatrix}
-1 & -1 \\
-1 & -1 \\
\end{pmatrix}
$$

Trivially, no channels are safe. Also, both channels are decreasing at layer $1$, with all negative values in $\mathbf{A}_{2}$. Thus condition (2) of \Cref{def:updown} is satisfied and both channels are increasing at layer $2$.

Next, to show that there exist unsafe channels that are stable, consider a 2-layer sum-GNN with the the following matrices. Let $\mathbf{B}_{1}^c = \mathbf{B}_{2}^c$ be $2 \times 2$ matrices filled with zeroes. Define $\mathbf{A}_{1}$ and $\mathbf{A}_{2}$ as follows:

$$
\mathbf{A}_{1} = \begin{pmatrix}
0 & 0 \\
0 & 0 \\
\end{pmatrix},~~~~~~~
\mathbf{A}_{2} = \begin{pmatrix}
-1 & -1 \\
-1 & -1 \\
\end{pmatrix}
$$

Both channels are safe at layer $1$ and both are unsafe at layer $2$, since $\mathbf{A}_{2}$ only contains non-zero values. Also, both channels are stable at layer $1$, with all zero values in $\mathbf{B}_{1}^c = \mathbf{B}_{2}^c$. Thus, the conditions for stability in \Cref{def:updown} are satisfied and both channels are stable at layer $2$.

The above examples can be extended to an arbitrary number of layers $L \geq 2$, dimensions $\delta_0, ..., \delta_L$, and colour set $\col$ by using the following scheme:
\begin{enumerate}
\item For each $\ell \in \{ 1, ..., L \}$ and $c \in \col$, let $\mathbf{B}_{\ell}^c$ a matrix of zeroes.
\item Let every entry in the matrix $\mathbf{A}_{L}$ be $-1$.
\item For each $\ell \in \{ 1, ..., L - 1 \}$, define the matrix $\mathbf{A}_{\ell}$ as follows:

\begin{enumerate}
\item To obtain the case where all channels are both unsafe and increasing at $L$, let every entry in the matrix $\mathbf{A}_{\ell}$ be $-1$, as in the example given above.
\item To obtain the case where all channels are both unsafe and stable at $L$, let every entry in the matrix $\mathbf{A}_{\ell}$ be $0$, also as in the example given above.
\end{enumerate}

\end{enumerate}

\end{proof}

\subsection{Theorem \ref{thm:neginfline_sound}} \label{app:neginfline_sound_proof}
\begin{theorem*}
Let $\mathcal{N}$ be a sum-GNN as in Equation \eqref{eq:GNN}, where 
$\sigma_{\ell}$ is ReLU for each $1 \leq \ell \leq L-1$,
and the co-domain of $\sigma_L$ contains a number strictly less than the threshold $t$ of the classification function $\texttt{cls}_t$.
Then, each rule with head predicate $U_p$ corresponding to an unbounded channel $p$ is unsound for $\mathcal{N}$.
\end{theorem*}

\begin{proof}
Assume to the contrary that there exists some rule $r$ with head $U_p(x)$ that is sound for $T_\mathcal{N}$. Let $\nu$ be an arbitrary variable substitution grounding the body of $r$.
Since $p$ is unbounded, there exist colours $c_1, c_2, ..., c_L$ and a vector $\mathbf{y}_0$ such that $\mathbf{y}_L[p] < 0$. Let $\mathbf{y}_1, ..., \mathbf{y}_{L-1}$ be defined as in \Cref{def:neginfline}, and $\mathbf{y}_L = \mathbf{B}_L^{c_L} \mathbf{y}_{L-1}$.

For each $\ell \in  \{0, \dots, L-1\}$, let $\mathbf{z}_\ell = \mathbf{B}^{c_{\ell+1}}_{\ell+1} \mathbf{y}_{\ell}$.
Note that $\mathbf{y}_{\ell+1} = \text{ReLU}(\mathbf{z}_{\ell})$
for $\ell \in \{0, \dots, L-2\}$.

\begin{center}
\textbf{Dataset Construction}
\end{center}

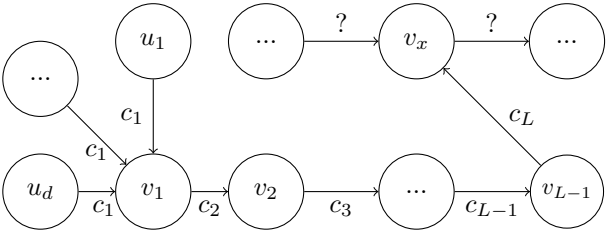
\begin{figure}
\begin{tikzpicture}
    \node[shape=circle,draw=black, minimum size =1cm] (u1) at (0, 11) {$u_1$};
    \node[shape=circle,draw=black, minimum size =1cm] (u) at (-1.5, 10.5) {$...$};
    \node[shape=circle,draw=black, minimum size =1cm] (ud) at (-1.5, 9) {$u_d$};
    \node[shape=circle,draw=black, minimum size =1cm] (v1) at (0, 9) {$v_1$};
    \node[shape=circle,draw=black, minimum size =1cm] (v2) at (1.5, 9) {$v_2$};
    \node[shape=circle,draw=black, minimum size =1cm] (v) at (3.5, 9) {$...$};
    \node[shape=circle,draw=black,scale=0.9] (vL) at (5.5, 9) {$v_{L-1}$};
    \node[shape=circle,draw=black, minimum size =1cm] (vx) at (3.5, 11) {$v_x$};
    \node[shape=circle,draw=black, minimum size =1cm] (b1) at (5.5, 11) {$...$};
    \node[shape=circle,draw=black, minimum size =1cm] (b2) at (1.5, 11) {$...$};

    \path [->] (u1) edge node[left] {${c_1}$} (v1);
    \path [->] (u) edge node[below] {${c_1}$} (v1);
    \path [->] (ud) edge node[below] {${\;\;c_1}$} (v1);
    \path [->] (v1) edge node[below] {${c_2}$} (v2);
    \path [->] (v2) edge node[below] {${c_3}$} (v);
    \path [->] (v) edge node[below] {${c_{L-1}}$} (vL);
    \path [->] (vL) edge node[right] {${\;c_{L}}$} (vx);
    \path [<-] (b1) edge node[above] {$?$} (vx);
    \path [->] (b2) edge node[above] {$?$} (vx);
\end{tikzpicture}
\caption{Succinct representation of the encoding of dataset $D_d$ in the family of datasets used in our counterexample.}
\label{fig:counterexample_in_proof}
\end{figure}

Furthermore, for each $d \in \mathbb{N} \cup \{0\}$, let 
$D_d$ be the dataset defined as the extension of $D^r_\nu$ with:
\begin{itemize}
\item $R^{c_1}(a_i,f_1)$ for $i \in \{1, \dots, d\}$,
\item $R^{c_\ell}(f_{\ell-1},f_{\ell})$
\end{itemize}
for $\ell \in \{2, \dots, L\}$, with $f_L=a=\nu(x)$,
(recall, we are assuming that we derive $U_p(a)$ for some channel $p$ and constant $a$). 
To keep our notation simple, for $d \in \mathbb{N} \cup \{0\}$ and $\ell \in \{0, \dots, L\}$, and constant $f_i$, let $v_i$ be the node in $\texttt{end}(D)$ associated with $f_i$ and likewise $u_i$ the node associated with $a_i$. Let $v_x$ be the node associated with $a$.

Let $\lambda^d_{\ell}(v_i)$ be the vector
assigned to $v_i$ in layer 
$\ell$ when applying $\mathcal{N}$ to the encoding of $D_d$. 
Furthermore, for each $\ell \in \{1, \dots, L-1\}$, let $\mathbf{t}_{\ell} = \lambda^0_{\ell-1}(v_\ell)$,
that is, the labelling assigned to $v_\ell$ in layer $\ell-1$ when applying 
$\mathcal{N}$ to the encoding of $D_0$. 
Observe that for each $d \geq 1$, $v_\ell$ is exactly $\ell$
steps away from each $u_i$, and so 
$\lambda^d_{\ell-1}(v_\ell)$
does not depend on $d$, that is, 
$\lambda^d_{\ell-1}(v_\ell) = 
\mathbf{t}_{\ell}$ for each $d \in \mathbb{N} \cup \{0\}$.

\begin{center}
\textbf{Claim to be Proved by Induction}
\end{center}

We next define inductively over $\ell \in \{1, \dots, L-1\}$
a sequence of natural numbers
$d_1 \leq d_2 \leq \dots \leq d_{L-1}$, and simultaneously
we prove the following property for $d_\ell$: 

\textbf{(Claim)} There exists a vector $\mathbf{g}_{\ell}$ of dimension $\delta_{\ell}$ such that, for each $d \geq d_{\ell}$, it holds that

$$\lambda^d_{\ell}(v_\ell)= \mathbf{g}_{\ell} + d \cdot \mathbf{y}_{\ell}.$$

Note that $\mathbf{g}_{\ell}$ is fixed before quantifying over $d$, and so it is independent of $d$.

To provide a intuitive summary about what we are using each of our variables to track:
\begin{itemize}
    \item Each $\mathbf{y}_{\ell}$ is a vector from \Cref{def:neginfline};
    \item $\mathbf{z}_{\ell}$ refers to the value of $\mathbf{y}_{\ell + 1}$ before the ReLU is applied;
    \item $v_i$ refers to the node in $\texttt{enc}(D_d)$ corresponding to $f_i$;
    \item $\lambda_\ell^d(v_i)$ refers to the label of $v_i$ in the graph $\texttt{enc}(D_d)$ after $\ell$ layers of $\mathcal{N}$;
    \item $\mathbf{t}_{\ell} = \lambda^0_{\ell-1}(v_\ell)$ is the labelling assigned to $v_\ell$ in layer $\ell-1$ when applying $\mathcal{N}$ to the encoding of $D_0$; and
    \item $\mathbf{g_\ell}$ refers to the non-$d$ component of $\lambda^d_{\ell}(v_\ell)$.
\end{itemize}

\begin{center}
\textbf{Base Case}
\end{center}

For the base case, we define $d_1$ as the smallest natural number such that 
for each $i \in \{1, \dots, \delta_1 \}$ with 
$\mathbf{z}_{0}[i] \neq 0$, it holds that 
$$d_{1} \cdot |\mathbf{z}_{0}[i]| > |\textbf{b}_1[i] + (\mathbf{A}_{1} \mathbf{t}_{1}) [i]|.$$
Clearly, such number always exists, since $\mathbf{z}_{0}[i] \neq 0$ implies 
$|\mathbf{z}_{0}[i]| > 0$. 

Now we prove the claim for $\ell=1$.
First, we have that for each $d$, $\lambda^d_{1}(v_1) = \text{ReLU}(\mathbf{b}_1 + \mathbf{A}_{1} \mathbf{t}_{1} + d \cdot \mathbf{z}_0)$.
Note that $\mathbf{t}_{1}$ has all elements equal to $0$ 
(since $\lambda^0_{0}(v_1)$ does by definition), so we ignore this term.
Now, consider an arbitrary $i \in \{1, \dots, \delta_1\}$.
We define $\mathbf{g}_1[i]$ attending to three possible cases:
\begin{itemize}
\item If $\mathbf{z}_0[i]$ is negative, we let $\mathbf{g}_1[i] = 0$.
To see that the claim holds, consider an arbitrary $d \geq d_1$.
Since $\mathbf{z}_0[i]$ is negative, 
$\mathbf{z}_0[i] \neq 0$; then, the definition of $d_1$ ensures that  
$\mathbf{b}_1[i] + d_1 \cdot \mathbf{z}_0[i] < 0$.
Since $d \geq d_1$, we then have 
$\mathbf{b}_1[i] + d \cdot \mathbf{z}_0[i] <0$ and so $\lambda^d_{1}(v_1)=0$.
However, $\mathbf{y}_{1}[i] =  \text{ReLU}(\mathbf{z}_0[i])=0$. Thus,
the claim holds.
\item If $\mathbf{z}_0[i] =0$, we let 
$\mathbf{g}[i] = \text{ReLU}(\mathbf{b}_1[i])$.
To prove the claim, assume again $d \geq d_1$. Since $\mathbf{z}_0[i] =0$,
$\lambda^d_{1}(v_1)[i]=\text{ReLU}(\mathbf{b}_\ell[i])$, so
the claim holds.
\item If $\mathbf{z}_0[i]$ is positive, then $\mathbf{g}_1[i]=\mathbf{b}_1[i]$.
To show the claim, consider $d \geq d_1$. Since $\mathbf{z}_0[i]$ is positive,
$\mathbf{z}_0[i] \neq 0$, and 
the definition of $d_1$ ensures that $\mathbf{b}_1[i] + d_1 \cdot \mathbf{z}_0[i] > 0$, but since $d \geq d_1$,
we have $\mathbf{b}_1[i] + d \cdot \mathbf{z}_0[i] > 0$, and so 
$\lambda_1^d(v_1)[i] = \mathbf{b}_1[i] + d \cdot \mathbf{z}_0[i]$.
Furthermore, $\mathbf{z}_0[i]$ being positive implies $\mathbf{y}_1[i] = \mathbf{z}_0[i]$, and so the claim holds.
\end{itemize}

\begin{center}
\textbf{Inductive Step}
\end{center}

Consider $\ell \in \{2, \dots, L-1\}$; suppose that 
$d_{\ell-1}$ is defined and the corresponding claim
holds, that is, there exists $\mathbf{g}_{\ell-1}$
such that for each $d \geq d_{\ell-1}$, it 
holds that 

$$ \lambda_{\ell-1}^d(v_{\ell-1}) = \mathbf{g}_{\ell-1} + d_{\ell-1} \cdot \mathbf{y}_{\ell-1}. $$

We now define $d_\ell$ as the smallest natural number greater than or equal to $d_{\ell-1}$ such that 
for each $i \in \{1, \dots, \delta_\ell\}$ with 
$\mathbf{z}_{\ell}[i] \neq 0$, it holds that 

$$d_{\ell} \cdot |\mathbf{z}_{\ell - 1}[i]| > |\textbf{b}_\ell[i] + (\mathbf{A}_{\ell} \mathbf{t}_{\ell})[i] + (\mathbf{B}^{c_\ell}_{\ell} \mathbf{g}_{\ell-1})[i] |.$$

Clearly, such a number always exists, since $\mathbf{z}_{\ell - 1}[i] \neq 0$
implies $|\mathbf{z}_{\ell - 1}[i]| >0$. 

We next prove the corresponding claim, that is, 
there exists $\mathbf{g}_{\ell}$ such that 
for arbitrary $d \geq d_{\ell}$, 

$$\lambda_{\ell}^d(v_{\ell}) = \mathbf{g}_{\ell} + d \cdot \mathbf{y}_{\ell},$$

for a fixed $\mathbf{g}_{\ell}$.
To define it first note that 
the vertices $u_1, ..., u_d$ are exactly $\ell$ steps away
from $v_\ell$, so such vertices do not affect the computation of
$\lambda_{\ell - 1}^d(v_{\ell})$; using this observation and 
the induction hypothesis (since $d_\ell \geq d_{\ell - 1}$) we have that for any $d \geq d_\ell$, 

$$\lambda_{\ell}^d(v_{\ell}) = \text{ReLU}(\mathbf{b}_\ell + \mathbf{A}_{\ell} \mathbf{t}_{\ell} + \textbf{B}^{c_\ell}_\ell ( \mathbf{g}_{\ell-1} + d \cdot \mathbf{y}_{\ell-1}),$$

which, using distribution of matrix product over vectors, 
can be rewritten as 
$$\lambda_{\ell}^d(v_{\ell}) = \text{ReLU}(\mathbf{b}_\ell + \mathbf{A}_{\ell} \mathbf{t}_{\ell} + \textbf{B}^{c_\ell}_\ell \mathbf{g}_{\ell-1} + d \cdot \mathbf{z}_{\ell-1}).$$
We next consider an arbitrary $i \in \{1, \dots,\delta_\ell\}$ and
define $\mathbf{g}_\ell[i]$, again attending to the sign of 
$\mathbf{z}_{\ell-1}[i]$.

\begin{itemize}
\item If $\mathbf{z}_{\ell-1}[i]$ is negative, then $\mathbf{g}_\ell[i] = 0$.
To show the claim let $d \geq d_{\ell}$.
Then $\mathbf{b}_\ell[i] + (\mathbf{A}_{\ell} \mathbf{t}_{\ell})[i] + (\mathbf{B}^{c_\ell}_\ell \mathbf{g}_{\ell-1})[i] + d_{\ell} \cdot \mathbf{z}_{\ell-1}[i] < 0$, by definition of $d_{\ell}$; however, since $d\geq d_\ell$, we have
$\mathbf{b}_\ell[i] + (\mathbf{A}_{\ell} \mathbf{t}_{\ell})[i] + (\mathbf{B}^{c_\ell}_\ell \mathbf{g}_{\ell-1})[i] + d \cdot \mathbf{z}_{\ell-1}[i] < 0$, and so 
$\lambda_{\ell}^d(v_{\ell})[i]=0$.
However, $\mathbf{y}_{\ell}[i] =  \text{ReLU}(\mathbf{z}_{\ell-1}[i])=0$. Thus,
the claim holds.

\item If $\mathbf{z}_{\ell-1}[i] =0$, we let 
$\mathbf{g}[i] = \text{ReLU}(\mathbf{b}_\ell[i] + (\mathbf{A}_{\ell} \mathbf{t}_{\ell})[i] + (\textbf{B}^{c_\ell}_\ell \mathbf{g}_{\ell-1})[i] )$.
To prove the claim, assume again $d \geq d_\ell$. Since $\mathbf{z}_{\ell-1}[i] =0 = \mathbf{y}_\ell[i]$, we have
$\lambda^d_{\ell}(v_\ell)[i]=\mathbf{g}_\ell[i]$,
and the claim holds. 

\item If $\mathbf{z}_{\ell-1}[i]$ is positive, then $\mathbf{g}_{\ell}[i]=\mathbf{b}_\ell[i] + (\mathbf{A}_{\ell} \mathbf{t}_{\ell})[i] + (\mathbf{B}^{c_\ell}_\ell \mathbf{g}_{\ell-1})[i]$.
To prove the claim, consider $d \geq d_{\ell}$. The definition of $d_\ell$ ensures that $\mathbf{b}_\ell[i] + (\mathbf{A}_{\ell} \mathbf{t}_{\ell})[i] + (\mathbf{B}^{c_\ell}_\ell \mathbf{g}_{\ell-1})[i] + d_{\ell} \cdot \mathbf{z}_{\ell-1}[i] > 0$, but since $d \geq d_\ell$, we have 
$\mathbf{b}_\ell[i] + (\mathbf{A}_{\ell} \mathbf{t}_{\ell})[i] + (\mathbf{B}^{c_\ell}_\ell \mathbf{g}_{\ell-1})[i] + d \cdot \mathbf{z}_{\ell-1}[i] > 0$, and so 
$\lambda_{\ell}^d(v_{\ell})[i] = \mathbf{b}_\ell[i] + (\mathbf{A}_{\ell} \mathbf{t}_{\ell})[i] + (\mathbf{B}^{c_\ell}_\ell \mathbf{g}_{\ell-1})[i] + d \cdot \mathbf{z}_{\ell-1}[i] > 0$.
Furthermore, $\mathbf{z}_{\ell-1}[i]$ being positive implies $\mathbf{y}_\ell[i] = \mathbf{z}_{\ell-1}[i]$, and so the claim holds.
\end{itemize}

This concludes the proof by induction.

\begin{center}
\textbf{Final $d_L$}
\end{center}

From the proof by induction, there exists $\mathbf{g}_{L-1}$
such that for each $d \geq d_{L-1}$, it 
holds that 

$$ \lambda_{L-1}^d(v_{L-1}) = \mathbf{g}_{L-1} + d_{L-1} \cdot \mathbf{y}_{L-1}. $$

Recall that since channel $p$ is unbounded, we have $\mathbf{y}_L[p] < 0$. Also there exists $\tau \in \mathbb{R}$ such that $\sigma_L(\tau) < t$, where $t$ refers to the threshold of the classifier $\texttt{cls}_t$ of $\mathcal{N}$.
We now define $d_L$ as the smallest natural number greater than or equal to $d_{L-1}$ such that

$$d_{L} \cdot \mathbf{y}_{L}[p] < -|(\textbf{b}_L + \mathbf{A}_{L} \mathbf{t}_{L} + \mathbf{B}^{c_L}_{L} \mathbf{g}_{L-1} + \mathbf{x}_L)[p] | + \tau ,$$

where $\mathbf{x}_L$ refers to the value passed from the other neighbours of $v_x$ (not including $v_{L-1}$). Note that this value does not depend on $d$.

Clearly, such a number always exists, since $\mathbf{y}_{L}[p] < 0$. 
Then $(d_{L} \cdot \mathbf{y}_{L} + \textbf{b}_L + \mathbf{A}_{L} \mathbf{t}_{L} + \mathbf{B}^{c_L}_{L} \mathbf{g}_{L-1} + \mathbf{x}_L)[p] < \tau$.
Then for any $d \geq d_L$, the above inequality still holds.
Also since $\sigma_L$ is monotonically increasing, we can derive $\sigma_L( d \cdot \mathbf{y}_{L} + \textbf{b}_L + \mathbf{A}_{L} \mathbf{t}_{L} + \mathbf{B}^{c_L}_{L} \mathbf{g}_{L-1} + \mathbf{x}_L )[p] < \sigma_L(\tau)$.

But notice that $\lambda_L^d(v_x) = \sigma_L( d \cdot \mathbf{y}_{L} + \textbf{b}_L + \mathbf{A}_{L} \mathbf{t}_{L} + \mathbf{B}^{c_L}_{L} \mathbf{g}_{L-1} + \mathbf{x}_L )$,
so $\lambda_L^d(v_x)[p] < \sigma(\tau) < t$.

\begin{center}
\textbf{Conclusion}
\end{center}

Fix $d = d_L$. We now have our counterexample dataset $D_d$.
Since $\lambda_L^d(v_x)[p] < t$, $U_p(a) \not\in T_\mathcal{N}(D_d)$.
But since $D_\nu^r \subseteq D_d$, $U_p(a) \in T_r(D_d)$, so $T_r(D_d) \not\subseteq T_\mathcal{N}(D_d)$.
Thus, $r$ is not sound for $T_\mathcal{N}$.
So, by contradiction, no rule mentioning $U_p$ in the head is sound for $T_\mathcal{N}$.

\end{proof}

\subsection{Theorem \ref{thm:unbounded_channel_relations}} \label{app:unbounded_channel_relations_proof}
\begin{theorem*}
\changemarker{
Unbounded channels are neither increasing, nor stable, nor safe. There exist, however, decreasing unbounded channels and undetermined unbounded channels.
}
\end{theorem*}

\begin{proof}
Consider a sum-GNN $\mathcal{N}$ as in \Cref{eq:GNN}. We prove the Theorem in two parts.

\begin{center}
\textbf{Unbounded channels are neither increasing, stable, nor safe}
\end{center}

Consider a channel $i \in \{ 1, ..., \delta_L \}$ at layer $L$. Assume to the contrary that $i$ is unbounded, and that it is also increasing or stable at layer $L$.

Let $\mathbf{y}_1, ..., \mathbf{y}_{L-1}$ be the sequence arising from \Cref{def:neginfline}. Since $i$ is unbounded, we have that $(\mathbf{B}_L^{c_L} \mathbf{y}_{L-1})[i] <0$. But since $\mathbf{y}_{L-1}[j] \geq 0$ for each $j \in \{ 1, ..., \delta_{L-1} \}$ (due to the application of ReLU), we have $\mathbf{B}_L^{c_L}[k] < 0$ for some $k \in \{ 1, ..., \delta_L \}$.

From \Cref{def:updown}, this implies that $i$ is not stable at layer $L$. Furthermore, from condition (4) of \Cref{def:updown}, $i$ is not increasing at layer $L$. This is a contradiction and thus proves that an unbounded channel cannot be increasing nor stable.

Finally, from \Cref{thm:channel_relations}, safe channels are increasing or stable. Thus, an unbounded channel also cannot be safe, proving the claim.

\begin{center}
\textbf{There exist decreasing unbounded channels and undetermined unbounded channels}
\end{center}
For simplicity, we provide examples of sum-GNNs with such unbounded channels using 2-layer GNNs with $\delta_0 = \delta_1 = \delta_2 = 2$ and $|\col| = 1$. However, the examples can be extended to sum-GNNs with an arbitrary number of layers, arbitrary $\delta$, and arbitrary number of colours.

Let $c \in \col$ be the unique colour in $\col$.
First, to show that there exist unbounded channels that are decreasing, consider a 2-layer sum-GNN with the the following matrices. Let $\mathbf{A}_{1} = \mathbf{A}_{2}$ be $2 \times 2$ matrices filled with zeroes. Define $\mathbf{B}_{1}^c$ and $\mathbf{B}_{2}^c$ as follows:

$$
\mathbf{B}_{1}^c = \begin{pmatrix}
1 & 0 \\
0 & 1 \\
\end{pmatrix},~~~~~~~
\mathbf{B}_{2}^c = \begin{pmatrix}
-1 & 0 \\
0 & 0 \\
\end{pmatrix}
$$

Then for $\mathbf{y}_0 = ( 1,~ 1 )^\top$, $\mathbf{y}_1 := \text{ReLU} ( \mathbf{B}_1^{c} \mathbf{y}_{0} ) = ( 1,~ 1 )^\top$ and $\mathbf{B}_2^{c} \mathbf{y}_{1} = ( -1,~ 0 )^\top$, so channel $1$ is unbounded.
Also, all channels are increasing at layer $1$, so channel $1$ is decreasing at layer $2$, providing the required example.

Next, to show that there exist unbounded channels that are undetermined, again consider a 2-layer sum-GNN with the the following matrices. Let $\mathbf{A}_{1} = \mathbf{A}_{2}$ be $2 \times 2$ matrices filled with zeroes. Define $\mathbf{B}_{1}^c$ and $\mathbf{B}_{2}^c$ as follows:

$$
\mathbf{B}_{1}^c = \begin{pmatrix}
1 & 0 \\
0 & 1 \\
\end{pmatrix},~~~~~~~
\mathbf{B}_{2}^c = \begin{pmatrix}
-2 & 1 \\
0 & 0 \\
\end{pmatrix}
$$

Then for $\mathbf{y}_0 = ( 1,~ 1 )^\top$, $\mathbf{y}_1 := \text{ReLU} ( \mathbf{B}_1^{c} \mathbf{y}_{0} ) = ( 1,~ 1 )^\top$ and $\mathbf{B}_2^{c} \mathbf{y}_{1} = ( -1,~ 0 )^\top$, so channel $1$ is unbounded.
Also, all channels are increasing at layer $1$, so channel $1$ is undetermined at layer $2$, providing the required example.

The above examples can be extended to an arbitrary number of layers $L$, dimensions $\delta_0, ..., \delta_L$, and colour set $\col$ by using the following scheme:
\begin{enumerate}
\item Pick some colour $c_0 \in \col$.
\item For each $\ell \in \{ 1, ..., L \}$ and $c \in \col$ such that $c \not= c_0$, let $\mathbf{B}_{\ell}^c$ and $\mathbf{A}_{\ell}$ be matrices of zeroes.
\item For each $\ell \in \{ 1, ..., L - 1 \}$, let $\mathbf{B}_{\ell}^{c_0}$ be the identity matrix.
\item Let $\mathbf{B}_{L}^{c_0}$ consist entirely of zeroes, except for $\mathbf{B}_{L}^{c_0} [0, 0]$ and $\mathbf{B}_{L}^{c_0} [0, 1]$.

\begin{enumerate}
\item To obtain the case where channel $1$ is both unbounded and decreasing, let $\mathbf{B}_{L}^{c_0} [0, 0] = -1$ and $\mathbf{B}_{L}^{c_0} [0, 1] = 0$, as in the example given above.
\item To obtain the case where channel $1$ is both unbounded and undetermined, let $\mathbf{B}_{L}^{c_0} [0, 0] = -2$ and $\mathbf{B}_{L}^{c_0} [0, 1] = 1$, also as above.
\end{enumerate}

\end{enumerate}

\end{proof}
\clearpage

\section{More Details About Experiments}
Instructions for how to run our code and reproduce our results can be found in the ``README.md'' file in our code.

The parameters we use to generate all our novel LogInfer datasets are given in \Cref{tab:loginfer_params}.

\begin{table}
\centering
\begin{tabular}{lrrrr}
\toprule
Dataset & $k_1$ & $k_2$ & \#M & \#NM \\
\midrule
FB-superhier & 5 & 50000 & 237 & 0 \\
WN-superhier & 200 & 50000 & 237 & 0 \\
\midrule
WN-hier\_nmhier & 5, 20 & 2000, 2000 & 1 & 10 \\
WN-cup\_nmhier & 5, 20 & 2000, 2000 & 1 & 10 \\
FB-hier\_nmhier & 50, 200 & 200, 200 & 21 & 216 \\
FB-cup\_nmhier & 50, 200 & 200, 200 & 21 & 216 \\

\bottomrule
\end{tabular}
\caption{Generation parameters for the new LogInfer datasets. Commas separate values where a dataset uses two rule patterns, and show the respective parameter values for the two patterns. \#M and \#NM are the number of distinct head predicates reserved for monotonic and non-monotonic rules, respectively.}
\label{tab:loginfer_params}
\end{table}

In \Cref{tab:rule_samples}, we provide a sample sound rule for each dataset/model combination that had some sound rules.

\begin{table*}
\centering
\resizebox{2.11\columnwidth}{!}{
\begin{tabular}{ll|l}
\toprule
Dataset & Model & Sample Rule \\
\midrule

WN-hier
& MGCN & \_similar\_to(x,y) and \_instance\_hypernym(y,y) implies \_member\_meronym(y,y) \\

\midrule

WN-sym
& MGCN & \_has\_part(y,x) and \_derivationally\_related\_form(y,y) implies \_derivationally\_related\_form(y,y) \\

\midrule

WN-superhier
& MGCN & \_synset\_domain\_topic\_of(y,y) and \_member\_of\_domain\_usage(y,z) implies \_hypernym(y,y) \\

\midrule

FB-hier
& MGCN & /base/popstra/celebrity/dated./base/popstra/dated/participant(x,y) implies /base/popstra/celebrity/dated./base/popstra/dated/participant(y,x) \\

\midrule

FB-sym
& MGCN & /influence/influence\_node/influenced\_by(x,x) implies /influence/influence\_node/influenced\_by(x,x) \\

\midrule

FB-superhier
& MGCN & /award/award\_category/winners./award/award\_honor/award\_winner(x,y) implies /celebrities/celebrity/celebrity\_friends./celebrities/friendship/friend(x,y) \\

\midrule

FB237v1
& MGCN & /olympics/olympic\_games/sports(x,y) implies /base/aareas/schema/administrative\_area/administrative\_parent(x,y) \\

\midrule

NELLv1
& R-0 & concept:headquarteredin(y,y) and concept:acquired(y,y) implies concept:acquired(y,y) \\
& R-25 & concept:acquired(y,y) and concept:organizationheadquarteredincity(y,z) implies concept:topmemberoforganization(y,y) \\
& R-50 & concept:worksfor(x,x) implies concept:organizationhiredperson(x,x) \\
& MGCN & concept:agentcontrols(y,y) and concept:subpartoforganization(y,z) implies concept:agentcollaborateswithagent(y,y) \\

\midrule

WN18RRv1
& R-0 & \_also\_see(x,y) implies \_also\_see(y,x) \\
& R-25 & \_verb\_group(x,y) and \_has\_part(z,y) implies \_derivationally\_related\_form(x,y) \\
& R-50 & \_has\_part(y,y) and \_also\_see(y,z) implies \_also\_see(z,y) \\
& MGCN & \_derivationally\_related\_form(x,y) implies \_derivationally\_related\_form(y,x) \\

\midrule

WN-hier\_nmhier
& R-0 & \_synset\_domain\_topic\_of(x,y) implies \_instance\_hypernym(x,y) \\
& R-25 & \_synset\_domain\_topic\_of(x,y) and \_also\_see(y,y) implies \_verb\_group(x,y) \\
& R-50 & \_member\_meronym(x,x) implies \_hypernym(x,x) \\
& MGCN & \_hypernym(x,y) and \_has\_part(z,y) implies \_derivationally\_related\_form(z,y) \\

\midrule

WN-cup\_nmhier
& R-0 & \_synset\_domain\_topic\_of(y,y) and \_derivationally\_related\_form(y,y) implies \_has\_part(y,y) \\
& R-25 & \_member\_of\_domain\_region(x,x) implies \_derivationally\_related\_form(x,x) \\
& R-50 & \_similar\_to(y,x) and \_also\_see(y,y) implies \_instance\_hypernym(y,x) \\
& MGCN & \_also\_see(x,y) and \_member\_of\_domain\_region(y,z) implies \_also\_see(y,x) \\

\bottomrule
\end{tabular}
} 
\caption{Sample sound rules for the models that had some sound rules.}
\label{tab:rule_samples}
\end{table*}

\end{document}